\documentclass[10pt,twocolumn,letterpaper]{article}

\usepackage{cvpr}
\usepackage{times}
\usepackage{epsfig}
\usepackage{graphicx}
\usepackage{amsmath}
\usepackage{amssymb}

\usepackage[numbers]{natbib}
\usepackage{comment}
\usepackage{subcaption}
\usepackage{booktabs}
\usepackage{epsfig}
\usepackage{epstopdf}
\newtheorem{theorem}{Theorem}
\newtheorem{assumption}{Assumption}

\newtheorem{proof}{Proof}

\newtheorem{corollary}{Corollary}
\newtheorem{lemma}{Lemma}
\usepackage{enumitem}

\usepackage{algorithm,algorithmicx,algpseudocode}
\usepackage{algcompatible}
\usepackage{amsmath}

\usepackage[pagebackref=true,breaklinks=true,letterpaper=true,colorlinks,bookmarks=false]{hyperref}

 \cvprfinalcopy 

\ifcvprfinal\fi
\begin{document}

\title{  Large Batch Training Does Not Need Warmup}

\author{Zhouyuan Huo, Bin Gu, Heng Huang\\
 Department of Electrical and Computer Engineering, University of Pittsburgh, USA\\
{\tt\small zhouyuan.huo@pitt.edu}
}

\maketitle
\begin{abstract}
	Training deep neural networks using a large batch size has shown promising results and benefits many real-world applications. However, the optimizer converges slowly at early epochs and there is a gap between large-batch deep learning optimization heuristics and theoretical underpinnings.  In this paper, we propose a novel Complete Layer-wise Adaptive Rate Scaling (CLARS) algorithm for large-batch training. We also analyze the convergence rate of the proposed method by introducing a new fine-grained analysis of gradient-based methods. Based on our analysis,  we bridge the  gap and illustrate the theoretical insights for three popular large-batch training techniques, including linear learning rate scaling, gradual warmup, and layer-wise adaptive rate scaling. Extensive experiments demonstrate that the proposed algorithm outperforms gradual warmup technique by a large margin and defeats the convergence of the state-of-the-art large-batch optimizer  in training advanced deep neural networks (ResNet, DenseNet, MobileNet) on ImageNet dataset.
\end{abstract}

\section{Introduction}
Deep learning has made significant breakthroughs in many fields, such as computer vision \cite{he2016deep,he2017mask,krizhevsky2012imagenet,ren2015faster}, nature language processing \cite{devlin2018bert,hochreiter1997long,vaswani2017attention}, and reinforcement learning \cite{mnih2013playing,silver2017mastering}. Recent studies show that better performance can usually be achieved by training a larger neural network with a bigger dataset \cite{mahajan2018exploring,radford2019language}. Nonetheless, it is time-consuming to train deep neural networks, which limits the efficiency of deep learning research. For example, training ResNet50 on ImageNet with batch size $256$ needs to take about $29$ hours to obtain $75.3\%$ Top-1 accuracy on $8$ Tesla P100 GPUs \cite{he2016deep}. Thus, it is a critical topic to reduce the training time for the development of deep learning. Data parallelism is the most popular method to speed up the training process, where the large-batch data is split across multiple devices \cite{dean2012large,krizhevsky2014one,yadan2013multi}. However, the large-batch neural network training using conventional optimization techniques usually leads to bad generalization errors \cite{hoffer2017train,keskar2016large}.

Many empirical training techniques have been proposed for large-batch deep learning optimization.  \cite{goyal2017accurate} proposed to adjust the learning rate through linear learning rate scaling and gradual warmup. By using these two techniques, they successfully trained ResNet50 with a batch size of $8192$ on $256$ GPUs in one hour with no loss of accuracy. Most of the theoretical analysis about linear learning rate scaling consider stochastic gradient descent only \cite{lian2015asynchronous,lian2016comprehensive}. However, the theoretical analysis for the momentum method or Nesterov's Accelerated Gradient \cite{nesterov1983method} is still unknown.  Finding that the ratios of weight's $\ell_2$-norm to gradient's $\ell_2$-norm vary greatly among layers,  \cite{you2017scaling}  proposed the state-of-the-art large-batch optimizer Layer-wise Adaptive Rate Scaling (LARS) and scaled the batch size to $16384$ for training ResNet50 on ImageNet. However, LARS still requires warmup in early epochs of training and may diverge  if it is not  tuned properly.  

Above three techniques (linear learning rate scaling, gradual warmup, and LARS) are demonstrated to be very effective and have been applied in many related works reducing the training time of deep neural networks \cite{akiba2017extremely,jia2018highly,mikami2018imagenet,ying2018image,you2019large}.   In spite of the effectiveness of above training techniques, theoretical motivations behind these techniques are still open problems: ({I}) Why we need to increase the learning rate linearly as batch size scales up? ({II}) Why we use gradual warm at early epochs, does there exist an optimal warmup technique with no need to tune hyper-parameters? ({III}) Why we need to adjust the learning rate layer-wisely?

In this paper, we target to remove the warmup technique for large-batch training and  bridge the gap between large-batch deep learning optimization heuristics and theoretical underpins. We summarize our main contributions as follows:
\begin{enumerate}[leftmargin=0.3in]
	\setlength\itemsep{0em}
	\item  We propose a novel Complete Layer-wise Adaptive Rate Scaling (CLARS) algorithm for large-batch deep learning optimization.  Then, we introduce a new fine-grained analysis for gradient-based methods and prove that the proposed method is guaranteed to converge for non-convex problems. 
	\item We bridge the gap between heuristics and  theoretical analysis  for three  large-batch deep learning optimization techniques, including layer-wise adaptive rate scaling, linear learning rate scaling, and gradual warmup.
	\item  Extensive experimental results demonstrate that CLARS outperforms gradual warmup by a large margin and defeats the convergence of the state-of-the-art large-batch optimizer in training advanced deep neural networks (ResNet, DenseNet, MobileNet) on ImageNet dataset.
\end{enumerate}

\section{Preliminaries and Challenges}
\noindent \textbf{Gradient-Based Methods:} The loss function of a neural network is minimizing the average loss over a  dataset of $n$ samples:
\begin{eqnarray}
\min_{w \in \mathbb{R}^d} & & \{f(w) :=  \frac{1}{n} \sum\limits_{i=1}^n f_i(w) \},
\label{obj}
\end{eqnarray}
where $d$ denotes the dimension of the neural network.
Momentum-based methods have been widely used in deep learning optimization, especially computer vision, and obtain state-of-the-art results \cite{he2016deep,huang2016densely}.
According to \cite{nesterov1983method}, mini-batch Nesterov Accelerated Gradient (mNAG)  optimizes the problem (\ref{obj}) as follows:
\vspace{-0.1cm}
\begin{eqnarray}
v_{t+1} = w_{t} - \gamma \frac{1}{B} \sum\limits_{i \in I_t} \nabla f_i(w_t), \nonumber \\
w_{t+1} = v_{t+1} + \beta( {v}_{t+1} - {v}_{t} ),
\label{nag}
\end{eqnarray}
where $I_t$ is the mini-batch samples with  $|I_t| =B$, $\gamma$ is the learning rate, $\beta \in [0,1)$ is the momentum constant and $v$ is the momentum vector. When $\beta=0$, Eq.~(\ref{nag}) represents the procedures of mini-batch Gradient Descent (mGD). Learning rate $\gamma$ is scaled up linearly when batch size $B$ is large \cite{goyal2017accurate}. However, using a learning rate $\gamma$ for all layers may lead to performance degradation. 

\noindent \textbf{Layer-Wise Learning Rate Scaling:} 
To train  neural networks with large batch size,  \cite{you2017scaling} proposed Layer-Wise Adaptive Rate Scaling (LARS). Suppose a neural network has $K$ layers, we can rewrite  $w = \left[(w)_1, (w)_2,..., (w)_K\right]$ with $(w)_k \in \mathbb{R}^{d_k}$ and $d = \sum_{k=1}^K d_k$. The learning rate at layer $k$ is updated as follows:
\begin{eqnarray}
\gamma_{k} &=& \gamma_{scale} \times \eta \times \frac{\|(w_t)_k\|_2}{\left\|  \frac{1}{B} \sum_{i \in I_t} \nabla_k f_i(w_t) \right\|_2} , 
\label{lars}
\end{eqnarray}
where $\gamma_{scale} = \gamma_{base} \times \frac{B}{B_{base}}$ and $\eta=0.001$ in \cite{you2017scaling}. $\gamma_{base}$ and $B_{base}$ depends on model and dataset. For example, we set $\gamma_{base}=0.1$ and $B_{base}=128$ to train ResNet on CIFAR10. Although LARS works well in practice, there is little theoretical understanding about it and it  converges slowly or  even diverges in the beginning of training if warmup \cite{goyal2017accurate} is not used . 

\noindent \textbf{Conventional Analysis:} \cite{bottou2016optimization,ghadimi2016accelerated,yang2016unified} proved the convergence of mGD or mNAG for non-convex problems through following two Assumptions:
\begin{assumption}
	\label{ass_linear}
	\emph{\textbf{(Lipschitz Continuous Gradient)}} The gradient of $f$ is Lipschitz continuous with constant  $L_{g}$. For any $w, v \in \mathbb{R}^d$, it is satisfied that:
	$
	\left\| \nabla f(w) - \nabla f(w+v) \right\|_2 \leq  L_g  \|v\|_2.$
\end{assumption}
\vspace{-6pt}
\begin{assumption}
	\label{iq_lip_g}
	\emph{\textbf{(Bounded Variance)}} There exist constants $M_g> 0$ and $M_C> 0$,  for any $w \in \mathbb{R}^d$, it is satisfied that $\mathbb{E}	\left\| \nabla f_i(w) - \nabla f(w) \right\|_2^2  \leq M_g \mathbb{E}  \|\nabla f(w) \|_2^2 + M_C$.
\end{assumption}

\begin{theorem}[\cite{yang2016unified}]
 Under Assumptions \ref{ass_linear} and \ref{iq_lip_g}, let $f_{\inf}$ denote the minimum value of problem $f(w)$ and $M_g=0$.   As long as $\gamma \leq  \frac{1-\beta}{2L}$, the gradient norm $\min\limits_{t=0,1,...,T-1} \left\| \nabla f(w_t) \right\|_2^2 $ is guaranteed to converge at the rate of $O\left(\frac{1}{T\gamma} + {L_g \gamma M_C} \right)$.
 \label{them0}
\end{theorem}
From Theorem \ref{them0}, it is natural to know that the value of $\gamma$ should be lowered because of the term $O(L_g\gamma M_C )$, which is consistent with the learning rate decay practically. However, there are two {weaknesses} of the current convergence result:   {({I})} It cannot explain why layer-wise learning rate in \cite{you2017scaling} is useful when there is one $\gamma$ for all layers.   {({II})} Theoretical result doesnot show that warmup is required in the early stage of training.

\section{Complete Layer-Wise Adaptive Rate Scaling  and Fine-Grained Convergence Analysis}
\label{sec:clars}
In this section, we propose a novel Complete Layer-wise Adaptive Rate Scaling (CLARS) algorithm for large-batch deep learning optimization and a new fine-grained convergence analysis of gradient-based methods for non-convex problems.

\subsection{Complete Layer-Wise Adaptive Rate Scaling}
 Define $U \in \mathbb{R}^{d\times d}$ as a permutation matrix where every row and column contains precisely a single $1$ with $0$s everywhere else. Let $U = [U_1, U_2, ...,U_K]$ and $U_k$ corresponds to the parameters of layer $k$, the relation between $w$ and $w_k$ is $w = \sum_{k=1}^K U_kw_k$.
Let $\nabla_k f_i(w_t)$ denote the stochastic gradient with respect to the parameters at layer $k$ and $\gamma_k$ denote its learning rate.  Thus, Eq.~ (\ref{nag}) of mNAG with batch $I_t$ can be  rewritten as:
\begin{eqnarray}
\left\{\begin{matrix}
v_{t+1} &=& w_{t}  - \sum\limits_{k=1}^K  \gamma_k  U_k   \left( \frac{1}{B} \sum\limits_{i \in I_t}  \nabla_k f_i(w_t)\right) \\
w_{t+1} & = & v_{t+1} + \beta( {v}_{t+1} - {v}_{t} )
\end{matrix}\right..
\label{mnag}
\end{eqnarray}
At each iteration, the learning rate $\gamma_k$ at layer $k$ is updated using Complete Layer-wise Adaptive Rate Scaling (CLARS) as follows:
\begin{eqnarray}
\gamma_{k} = { \gamma_{scale} \times \eta  \times \frac{\|(w_t)_k\|_2}{ \frac{1}{B} \sum_{i \in I_t} \left\|  \nabla_k f_i(w_t)  \right\|_2 } },
\label{clars}
\end{eqnarray}
where $\gamma_{scale} = \gamma_{base} \times \frac{B}{B_{base}}$ and $\eta$ is constant. To obtain a clear understanding of Eq.~(\ref{clars}), we rewrite it as: 
\begin{eqnarray}
\gamma_{k} =  \gamma_{scale} \times \eta  \times \frac{\|(w_t)_k\|_2}{\left\| \frac{1}{B} \sum_{i \in I_t}  \nabla_k f_i(w_t)\right\|_2}   \hspace{0.15cm} \nonumber \\
\hspace{1cm} \times  \hspace{0.15cm}    \frac{\left\| \frac{1}{B} \sum_{i \in I_t}  \nabla_k f_i(w_t)\right\|_2}{  \frac{1}{B} \sum_{i \in I_t} \left\| \nabla_k f_i(w_t)\right\|_2  }.\nonumber
\end{eqnarray}
It is equal to multiplying the LARS learning rate in Eq.~(\ref{lars}) with a  new term $ \frac{\left\| \frac{1}{B} \sum_{i \in I_t}  \nabla_k f_i(w_t)\right\|_2}{  \frac{1}{B} \sum_{i \in I_t} \left\| \nabla_k f_i(w_t)\right\|_2}$, which plays a  critical role in removing the warmup. The proposed CLARS method  is briefly summarized in Algorithm \ref{alg_sgd}.

In the following section, we will show that CLARS  is supported theoretically and the learning rate at layer $k$ is normalized with respect to its corresponding Lipschitz constant and gradient variance. In the experiments, we will also demonstrate  that the proposed method can complete large-batch ImageNet training with no warmup for the first time and accelerate the convergence.

\vspace{0pt}
\begin{algorithm}[t]
	\renewcommand{\algorithmicrequire}{\textbf{Phase II:}}
	\renewcommand{\algorithmicensure}{\textbf{Require:}}
	\caption{ Complete Layer-Wise Adaptive Rate Scaling}
	\label{alg_sgd}
	\begin{algorithmic}[1]
		\ENSURE $\gamma_{scale}$: Maximum learning rate
		\ENSURE $\beta$: Momentum parameter
		\ENSURE $\eta=0.01$	
		\FOR{$t=0,1,2, \cdots,T-1$}
		\STATE Sample large-batch $I_t$ randomly with batch size $B$;
		\STATE Compute large-batch gradient $ \frac{1}{B} \sum_{i\in I_t} \nabla f_{i}(w_t)$;
		\STATE Compute the  average of gradient norm for $K$ layers $\frac{1}{B} \sum_{i\in I_t} \left\| \nabla_{k} \nabla f_i(w_t) \right\|_2^2 $;
			\STATE Update layer-wise learning rate $\gamma_{k}$ following Eq.~(\ref{clars});
			\STATE Update the model $w_t$ and momentum term $v_t$ following Eq.~(\ref{mnag});
		\ENDFOR
		\STATE  Output $w_{T}$ as the final result.
	\end{algorithmic}
\end{algorithm}

\label{sec:analysis}

\subsection{Fine-Grained Micro-Steps and Assumptions}
\label{sec:fine_step}
In this section, we propose a new fine-grained method for the convergence analysis of gradient-based methods. Based on the fine-grained analysis, we prove the convergence rate of mini-batch Gradient Descent (mGD) and mini-batch Nesterov's Accelerated Gradient (mNAG) for deep learning problems.  More insights are obtained by analyzing their convergence properties.

Each step of mNAG in Eq.~(\ref{mnag}) can be regarded as the result of updating $v, w$ for $K$ micro-steps, where the gradient at each micro-step is $\frac{1}{B} \sum_{i \in I_t} \nabla_k f_i(w_t)$. At micro-step $t$:$s$,  we have layer index $k(s)=s \pmod{K}+1$.  For example, when $s=0$, we are updating the parameters of layer $k(0) = 1$.
Defining $w_{t:0} = w_t$, $w_{t:K} = w_{t+1}$, we can obtain Eq.~(\ref{mnag}) after applying following equations from $s=0$ to $s=K-1$:
\begin{eqnarray}
\left\{\begin{matrix}
v_{t:s+1} &=& w_{t:s} -  \frac{\gamma_k}{B}  \sum\limits_{i \in I_t}  U_k   \nabla_k f_i(w_t) \\
w_{t:s+1} & = & v_{t:s+1} + \beta( {v}_{t:s+1} - {v}_{t:s} )
\end{matrix}\right. .
\label{newnag}
\end{eqnarray}

Following the idea of block-wise Lipschitz continuous assumption in  \cite{beck2013convergence} and regarding  layers as blocks, we suppose that  two layer-wise assumptions are satisfied for any $K$-layer neural network throughout this paper, . 
\begin{assumption}[Layer-Wise Lipschitz  Continuous Gradient]
	Assume that the gradient of $f$ is layer-wise Lipschitz continuous and the Lipschitz constant corresponding to layer $k$ is $L_k$ for any layer $k \in \{1,2,...,K\}$.  For any $w \in \mathbb{R}^d$ and $v = [v_1,v_2, ..., v_K]  \in \mathbb{R}^d$, the following inequality is satisfied that for any $k \in \{1,2,...,K\}$:
	\begin{eqnarray}
	\left\| \nabla_k f(w) - \nabla_k f(w+U_kv_k) \right\|_2 &\leq & L_k  \|v_k\|_2. \nonumber
	\end{eqnarray}
	\label{ass_lip}\vspace{-0.3cm}
\end{assumption}
Lipschitz constants $L_k$ of different layers are not equal and can be affected by multiple factors, for example, position (top or bottom) or layer type  (CNN or FCN).  \cite{zou2018lipschitz} estimated  Lipschitz constants empirically and verified that Lipschitz constants of gradients at different layers vary a lot. $L_k$ represents the property at layer $k$ and plays an essential role in tuning learning rates. In addition, we also think the ``global'' Lipschitz continuous assumption in Assumption \ref{ass_linear} is satisfied and $L_g \geq L_k$. 
\begin{assumption}[Layer-Wise Bounded Variance]
	\label{ass_bd}
Assume that the variance of stochastic gradient with respect to the parameters of layer $k$ is upper bounded. For any $k \in \{1,2,...,K\}$ and $w \in \mathbb{R}^d$,  there exists $M_k>0$ and $M>0$ so that:
	\begin{eqnarray}
	\mathbb{E}	\left\| \nabla_k f_i(w) - \nabla_k f(w) \right\|_2^2  \leq M_k \mathbb{E}  \|\nabla_k f(w) \|_2^2 + M. \nonumber
	\end{eqnarray}
\end{assumption}
Let $M_k \leq M_{g}$ for any $k$, it is straightforward to get the upper bound of the variance of gradient $\nabla f_i(w)$ as $\mathbb{E}	\left\| \nabla f_i(w) - \nabla f(w) \right\|_2^2  \leq M_g \mathbb{E}  \|\nabla f(w) \|_2^2 + KM$. It is obvious that the value of $M_C=KM$ in Assumption \ref{iq_lip_g} is dependent  on the  neural networks depth. 

\noindent \textbf{Difficulties of Convergence Analysis:} There are two major difficulties in proving the convergence rate using the proposed fine-grained micro-steps. {({I})} Micro-step induces stale  gradient in the analysis. At each micro-step $t$:$s$ in Eq.~(\ref{newnag}),  gradient is computed using  the stale model $w_t$, rather than the latest model $w_{t:s}$.    {({II})} $K$ Lipschitz constants for $K$ layers are considered separately and  simultaneously, which is much more complicated than just considering $L_g$ for the whole model.

\subsection{ Convergence Guarantees of Two  Gradient-Based Methods}

Based on the proposed fine-grained analysis, we prove that both of mini-batch Gradient Descent (mGD) and mini-batch Nesterov's Accelerated Gradient (mNAG) admit sub-linear convergence guarantee $O\left(\frac{1}{\sqrt{T}}\right)$  for non-convex problems. Finally, we obtain some new insights about the gradient-based methods by taking mNAG as an example.  At first, we let $\beta=0$ in Eq.~(\ref{mnag}) and Eq.~(\ref{newnag}), and analyze the convergence of mGD method.

\begin{theorem}[Convergence of mGD]
	Under Assumptions \ref{ass_lip} and \ref{ass_bd}, let $f_{\inf}$ denote the minimum value of problem $f(w)$, $\kappa_k = \frac{L_g}{L_k} \leq \kappa$, $\gamma_{k} = \frac{\gamma}{L_k}$, and $\sum_{k=1}^K q_k \mathbb{E} \left\|  \nabla_k f(w_t)  \right\|_2^2 $ represents the expectation of $\mathbb{E}\left\|  \nabla_k f(w_t)  \right\|_2^2$ with probability { $q_k=\frac{{1}/{L_k}}{\sum_{k=1}^K ({1}/{L_k})}$} for any $k$.  As long as  $\gamma_k \leq \min \left\{\frac{1}{8L_k}, \frac{B}{8 L_kM_k} \right\} $ and $ \frac{1}{K} \sum\limits_{k=1}^K\gamma_k \leq \min \left\{\frac{1}{2L_g},   \frac{1}{2 L_g }  \sqrt{\frac{B}{M_g  }} \right\} $, it is guaranteed that:
	\begin{eqnarray}
	\frac{1}{T}  \sum\limits_{t=0}^{T-1} \sum\limits_{k=1}^K q_k \mathbb{E} \left\|   \nabla_k f(w_t)  \right\|_2^2  &\leq & \frac{8(f(w_0)-f_{\inf} )}{T\gamma \sum\limits_{k=1}^K \frac{1}{L_k} } \nonumber\\
	&& + \frac{( 4+ 2\kappa )M\gamma}{B}.\nonumber
	\end{eqnarray}
	\label{them1}
\end{theorem}
Different from Theorem \ref{them0}, we use $\sum\limits_{k=1}^K q_k \mathbb{E} \left\|   \nabla_k f(w_t)  \right\|_2^2 $ to measure  convergence in the paper. Specially, if $L_k=L_g$ for all $k$, it is easy to know that $q_k = \frac{1}{K}$ for all $k$ and   $\sum\limits_{k=1}^K q_k \mathbb{E} \left\|   \nabla_k f(w_t)  \right\|_2^2 =  \frac{1}{K} \mathbb{E} \left\|   \nabla f(w_t)  \right\|_2^2$. 
From  Theorem \ref{them1}, we prove that mGD admits sub-linear convergence rate $O\left(\frac{1}{\sqrt{T}}\right)$ for non-convex problems. 
\begin{corollary}[Sub-Linear Convergence Rate of mGD]
  Theorem \ref{them2} is satisfied and follow its notations. Suppose $\frac{1}{8L_k}$ dominates the upper bound of $\gamma_k$, and let
$	\gamma =  \min \left\{ \frac{1}{8}, \sqrt{ \frac{B (f(w_0)-f_{\inf} )  }{TM \sum\limits_{k=1}^K \frac{1}{L_k} }  } \right\} $, mGD is guaranteed to converge that:
	\begin{eqnarray}
	\frac{1}{T}  \sum\limits_{t=0}^{T-1} \sum\limits_{k=1}^K q_k \mathbb{E} \left\|   \nabla_k f(w_t)  \right\|_2^2 
\leq   \frac{64(f(w_0)-f_{\inf} )}{T \sum\limits_{k=1}^K \frac{1}{L_k} } \nonumber \\
 +  \left( 12 + 2\kappa  \right) \sqrt{ \frac{M (f(w_0)-f_{\inf} )  }{TB \sum\limits_{k=1}^K \frac{1}{L_k}  }  }.  \nonumber 
\end{eqnarray}
	\label{cor_1_1}
\end{corollary}

So far, we have proved the convergence of mGD method for non-convex problems. When $\beta \neq 0$, we can also prove the convergence of mNAG as follows:

\begin{theorem}[Convergence of mNAG]
	Under Assumptions \ref{ass_lip} and \ref{ass_bd}, let $f_{\inf}$ denote the minimum value of problem $f(w)$,  $\kappa_k = \frac{L_g}{L_k} \leq \kappa$, $\gamma_{k} = \frac{\gamma}{L_k} $, and $\sum_{k=1}^K q_k \mathbb{E} \left\|  \nabla_k f(w_t)  \right\|_2^2 $ represents the expectation of $\mathbb{E}\left\|  \nabla_k f(w_t)  \right\|_2^2$ with probability { $q_k=\frac{{1}/{L_k}}{\sum_{k=1}^K ({1}/{L_K})}$} for any $k$. Therefore, as long as
$\gamma_k \leq \min\left\{\frac{(1-\beta)}{8L_k}, \frac{(1-\beta)B}{8 L_kM_k} \right\}$ and 
$
\frac{1}{K}	\sum\limits_{k=1}^K \gamma_k \leq \min\left\{ \frac{(1-\beta)^2}{4\beta^2 L_g }, \frac{(1-\beta)^2 \sqrt{B}}{4\beta^2 L_g \sqrt{ M_g} },  \frac{(1-\beta) \sqrt{B} }{4L_g  \sqrt{M_g}  }  , \frac{(1-\beta)}{4 L_g}   \right\}$,
	it is satisfied that:
	{\small
		\begin{eqnarray}
	\frac{1}{T}  \sum\limits_{t=0}^{T-1} \sum\limits_{k=1}^K q_k  \mathbb{E} \left\|\nabla_{k} f(w_t)  \right\|_2^2 \leq  \frac{8(1-\beta)(f(w_0) - f_{\inf} ) }{T \gamma\sum\limits_{k=1}^K \frac{1}{L_k}} \nonumber \\
	   + \frac{M \gamma}{(1-\beta)B} \left(4 +   2\kappa +  \frac{2 \kappa }{ (1-\beta) }  \right). \nonumber 
		\end{eqnarray}}
	\label{them2}
\end{theorem}
Similarly, we can easily prove that mNAG is guaranteed to converge for non-convex problems with a sub-linear rate $O\left( \frac{1}{\sqrt{T}} \right)$ as follows:
\begin{corollary}[Sub-Linear Convergence of mNAG]
 Theorem \ref{them2} is satisfied and follow its notations, 
	Suppose $\frac{1-\beta}{8L_k}$ dominates the upper bound of $\gamma_k$, if $	\gamma  = \min \left\{ \frac{1-\beta}{8}, \sqrt{ \frac{B (f(w_0)-f_{\inf} )  }{TM \sum\limits_{k=1}^K \frac{1}{L_k} }  } \right\}  $, mNAG is guaranteed to converge that:
	{\small
		\begin{eqnarray}
		\min\limits_{t \in \{0,...,T-1\}}  \sum\limits_{k=1}^K q_k  \mathbb{E} \left\|\nabla_{k} f(w_t)  \right\|_2^2  \leq    \frac{64(f(w_0)-f_{\inf} )}{(1-\beta)T \sum\limits_{k=1}^K \frac{1}{L_k} } + \bigg( 8  \nonumber\\
		 + \frac{ 1}{(1-\beta)} \left(4 +   2\kappa +  \frac{2 \kappa }{ (1-\beta) }  \right) \bigg) \sqrt{ \frac{M (f(w_0)-f_{\inf} )  }{TB \sum\limits_{k=1}^K \frac{1}{L_k}  }  }.
		\label{cor2_1_eq1}
		\end{eqnarray}
	}
	\label{cor2_1}
\end{corollary}

According to Theorem \ref{them2},  we know that  the result of Theorem \ref{them0} is a special case of Theorem \ref{them2} when $L_k=L_g$ and $M_k=M_g$.
\begin{corollary}[Convergence when $L_k=L_g$ and $M_k=M_g$]
	Suppose Theorem \ref{them2} is satisfied and follow its notations.  If $L_k=L_g$, and $M_k = M_g$, $M_C=KM$, we have $\kappa_{k}=1$, $\gamma_{g}=\gamma_k$. As long as the learning rate $ \gamma_g \leq \min \left\{\frac{1-\beta}{8L_g}, \frac{B(1-\beta)}{8 L_gM_g}, \frac{(1-\beta)\sqrt{B} }{4L_g \sqrt{ M_g } },\frac{(1-\beta)^2\sqrt{B} }{4\beta^2 L_g \sqrt{ M_g } }, \frac{(1-\beta) }{4 \beta^2 L_g  } \right\} $, it is guaranteed that:
	\begin{eqnarray}
	\frac{1}{T}  \sum\limits_{t=0}^{T-1}  \mathbb{E} \left\|   \nabla f(w_t)  \right\|_2^2  \leq  \frac{8(1-\eta)(f(w_0)-f_{\inf} )}{T\gamma_g } \nonumber\\
	 + \frac{M_C L_g \gamma_g}{(1-\beta)} \left( 6 + \frac{2}{1-\beta} \right). 
	\label{iq_cor_2_3} 
	\end{eqnarray}
	\label{cor_2_3}
\end{corollary}
In Corollary \ref{cor_1_1} and \ref{cor2_1}, we ignore the upper bound of $\frac{1}{K}  \sum\limits_{k=1}^K \gamma_k$ for simplicity. It can be easily satisfied by making some $\gamma_{k}$ small.

\subsection{Discussions About  the Convergence of mNAG}
\label{sec4:dis}
According our fine-grained convergence analysis of gradient-based methods, we take mNAG as an example and gain more insights about the convergence of mNAG for  neural networks.

\textbf{Data Parallelism.}  Data parallelism is widely used in the training of deep learning models, and linear speedup can be obtained if learning rate and communication can be properly handled. Suppose that { $\min\limits_{t \in \{0,...,T_{\varepsilon}-1\}}   \sum\limits_{k=1}^K q_k  \mathbb{E} \left\|\nabla_{k} f(w_t)  \right\|_2^2\leq \varepsilon$} is satisfied after optimizing problem $f(w)$ using batch size $B$ after $T_{\varepsilon}$ iterations. Linear speedup means that when batch size scales up by $c\geq 1$ times ($B \rightarrow cB$),  we can obtain the same convergence guarantee $\varepsilon$ after only $\frac{T_{\varepsilon}}{c}$ iterations ($T_{\varepsilon} \rightarrow \frac{T_{\varepsilon}}{c}$).   
From Corollary \ref{cor2_1}, if $\gamma$ is dominated by $\sqrt{ \frac{B^2 (f(w_0)-f_{\inf} )  }{TBM \sum\limits_{k=1}^K \frac{1}{L_k} }}$,  the left term in Eq.~ (\ref{cor2_1_eq1}) converges with a rate of  $O(\sqrt{ \frac{1 }{TB  }  })$. It is guaranteed to converge to the same error as long as $TB$ is fixed. Therefore, we know that  when $B$ is scaled up by $c$ times to $cB$, the problem can converge to the same error after $\frac{T}{c}$ iterations, as long as $\gamma$ is also scaled up by $B$ times.  

\textbf{Lipschitz Constant Scaled Learning Rate.}
From Theorem \ref{them2}, the learning rate at  layer $k$ is computed through $\gamma_k = \frac{\gamma}{L_k}$. It offers us a method to tune $K$ learning rates $\gamma_k$ for a $K$-layer neural network simultaneously using just one parameter $\gamma$.

\textbf{Layer-Wise Model Scaling Factor $\kappa_k$.}
Define $\kappa_k = \frac{L_g}{L_k} \geq 1$  as the scaling factor at layer $k$. Because of the upper bound of $\gamma_k \leq \min\left\{\frac{(1-\beta)}{8L_k}, \frac{(1-\beta)B}{8 L_kM_k} \right\}$ in Theorem \ref{them1},  we know that designing a layer with larger $\kappa_{k}$ can increase the upper bound of learning rate at layer $k$.  In \cite{santurkar2018does}, authors show that batch normalization can help to increase $\kappa_k$.

\textbf{Layer-Wise Gradient Variance Factor $M_k$.}
Define $M_k$ as the gradient variance factor at layer $k$, which is dependent on the data and the model, and varies in the process of training.  Because of the upper bound of $\gamma_k \leq \min\left\{\frac{(1-\beta)}{8L_k}, \frac{(1-\beta)B}{8 L_kM_k} \right\}$ in Theorem \ref{them1}, it shows that  batch size $B$ can be scaled up as long as $B\leq M_k$. Therefore, a larger $M_k$ helps the algorithm obtain faster speedup. In the following section, we will show that warmup is closely related to $M_k$.

\section{Experimental Results}
\label{sec:three}
In this section, we conduct experiments to validate our convergence results empirically and demonstrate the superior performance of CLARS method over  LARS method. 
Firstly, we evaluate the necessity of using LARS on training neural networks. Secondly,  we  verify linear learning rate scaling theoretically and empirically.  Thirdly,  we propose one hypothesis about the reason of warmup and visualize it. Finally, extensive experiments are conducted to show that CLARS can replace warmup trick completely and converges faster than LARS with fine-tuned warmup steps.   All experiments are implemented in PyTorch 1.0 \cite{paszke2017automatic} with Cuda v10.0 and performed on a machine  with Intel(R) Xeon(R) CPU E5-2683 v4 @ 2.10GHz and 4 Tesla P40 GPUs.

\begin{figure*}[t]
	\centering
	\vspace*{-5pt}
	\begin{subfigure}[b]{0.32\textwidth}
		\centering
		\includegraphics[width=2.36in]{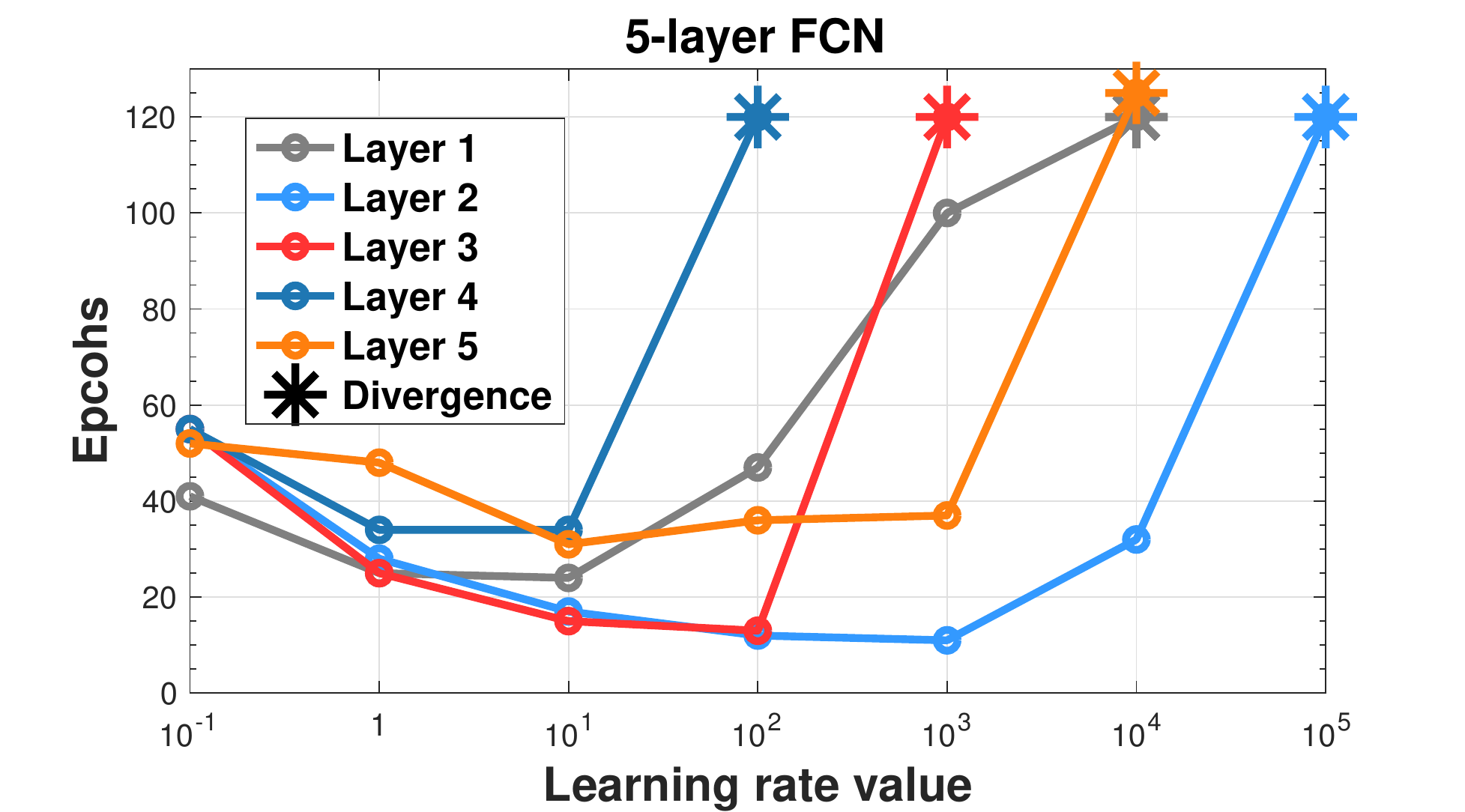}
	\end{subfigure}
	\begin{subfigure}[b]{0.32\textwidth}
		\centering
		\includegraphics[width=2.36in]{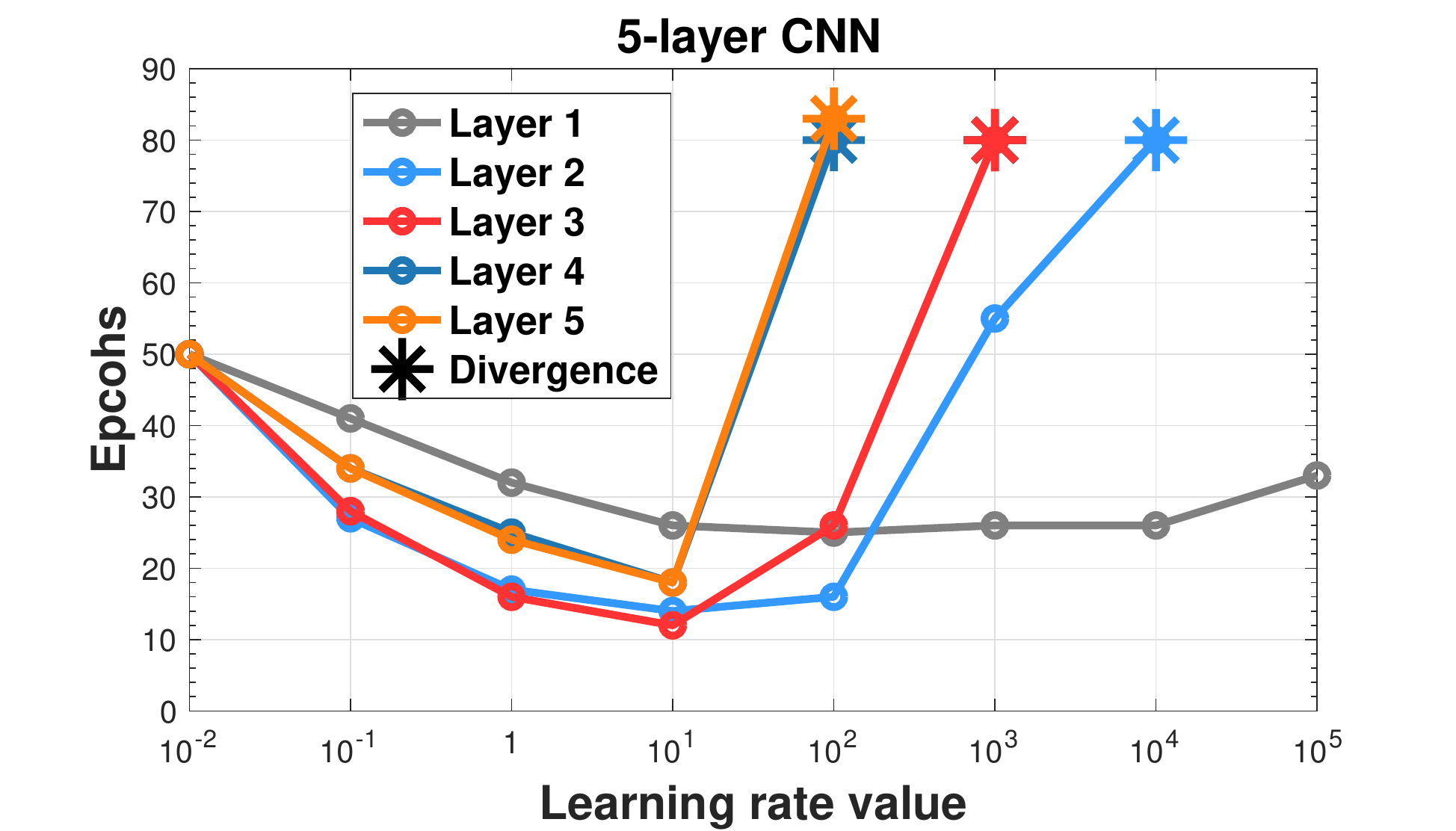}
	\end{subfigure}
	\begin{subfigure}[b]{0.32\textwidth}
		\centering
		\includegraphics[width=2.36in]{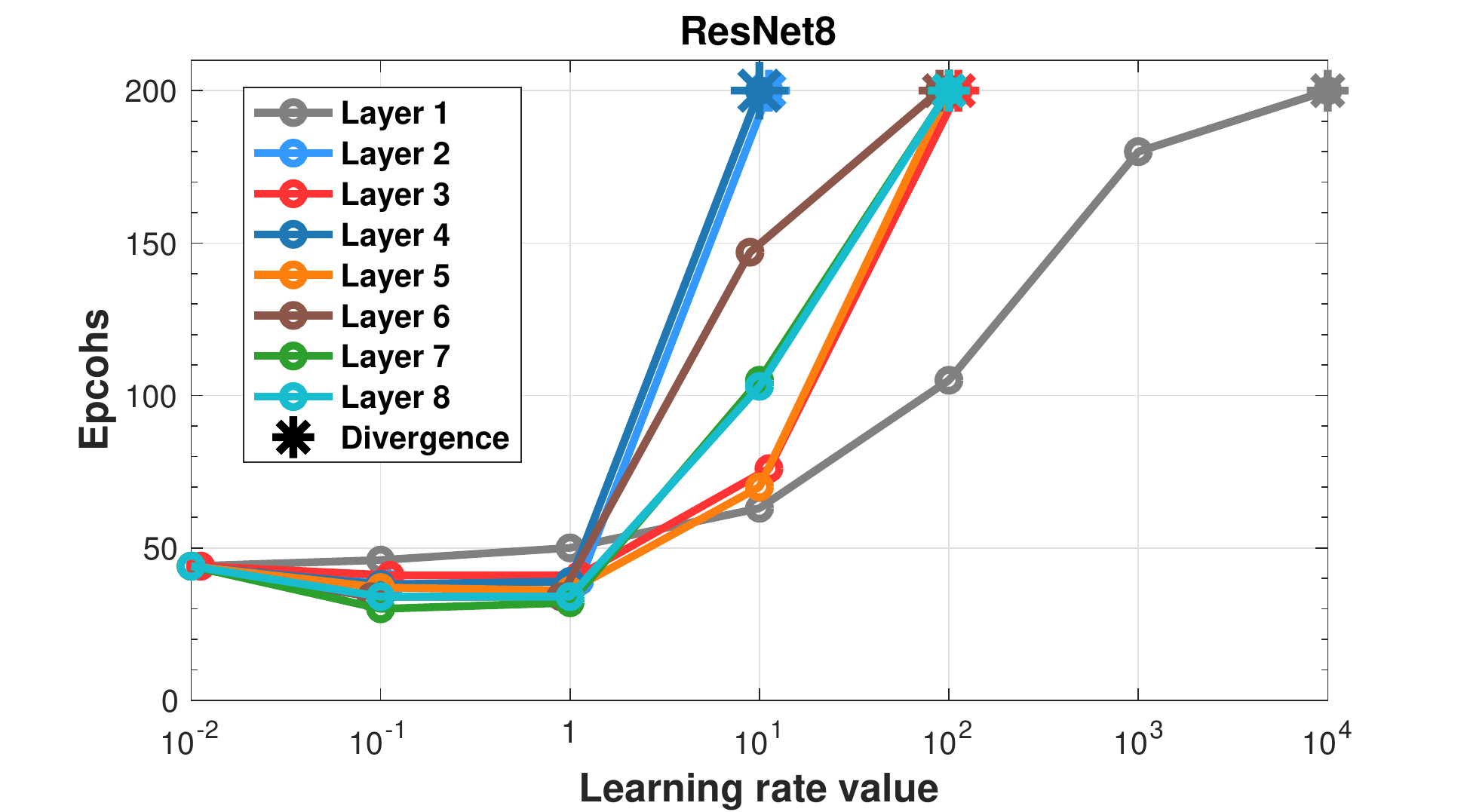}
	\end{subfigure}
	\caption{Learning rate upper bound of each layer. We train  $5$-layer FCN and $5$-layer CNN with sigmoid activation on MNIST and count the epochs required reaching training loss $0.03$ and $0.02$ respectively. We train ResNet8 (no batch normalization layer) on CIFAR-10 and count the epochs required reaching training loss $1.0$. $\textbf{*}$ denotes that loss diverges using the corresponding learning rate. }
	\label{fig::lw}
	\vspace*{-5pt}
\end{figure*}

\begin{figure*}[t]
	\centering
	\begin{subfigure}[b]{0.32\textwidth}
		\centering
		\includegraphics[width=2.36in]{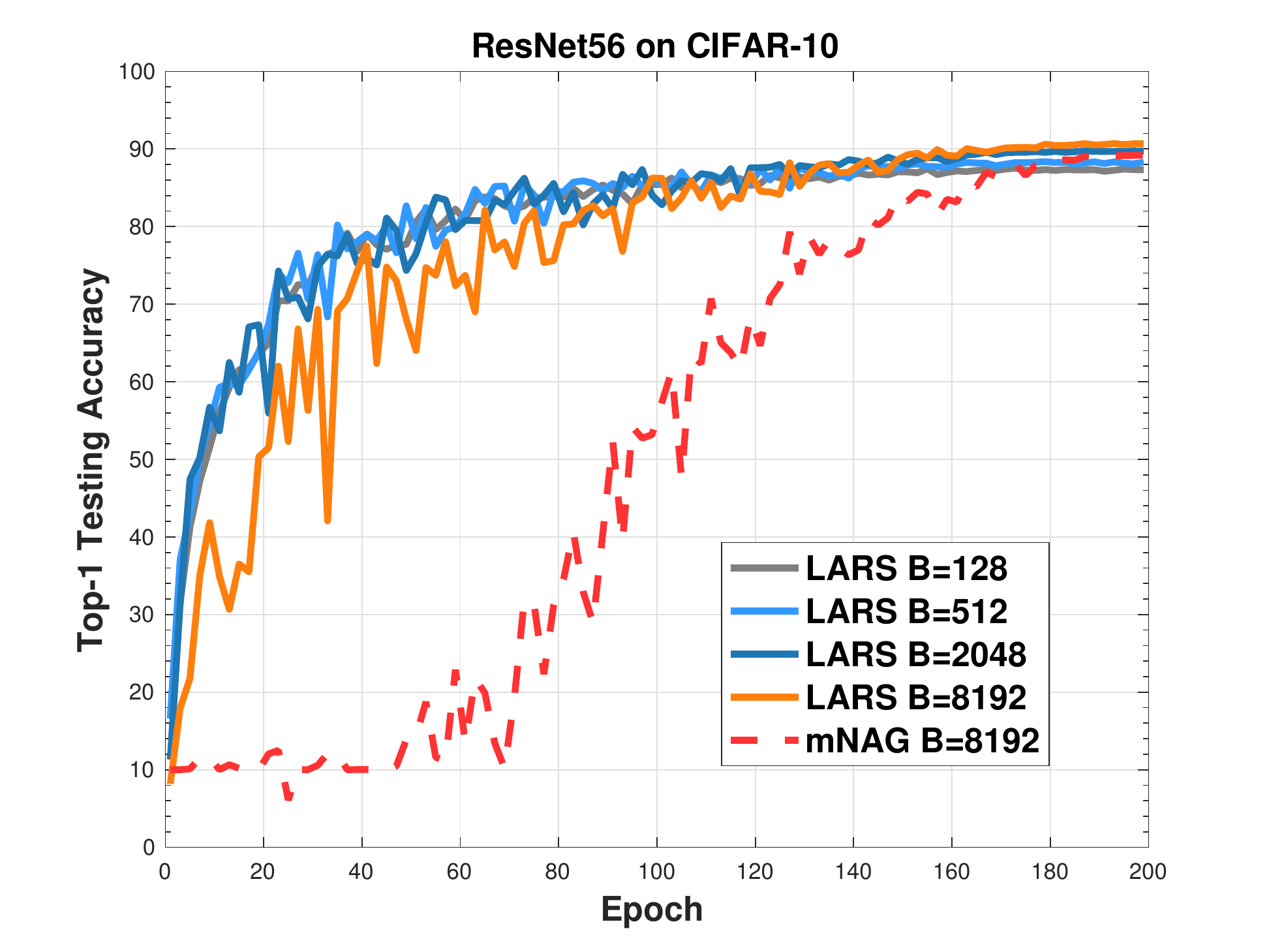}
	\end{subfigure}
	\begin{subfigure}[b]{0.32\textwidth}
		\centering
		\includegraphics[width=2.36in]{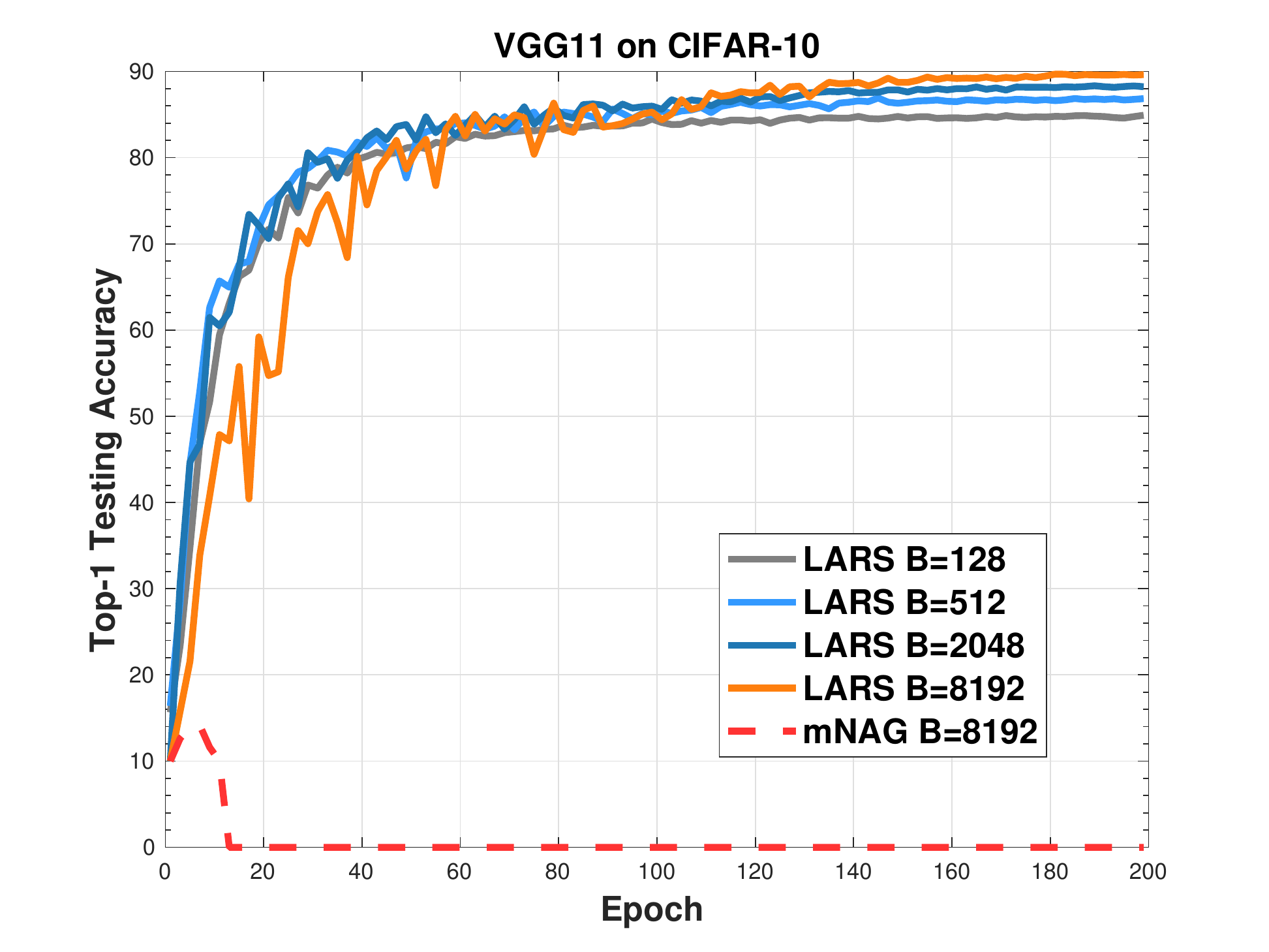}
	\end{subfigure}
	\begin{subfigure}[b]{0.32\textwidth}
		\centering
		\includegraphics[width=2.36in]{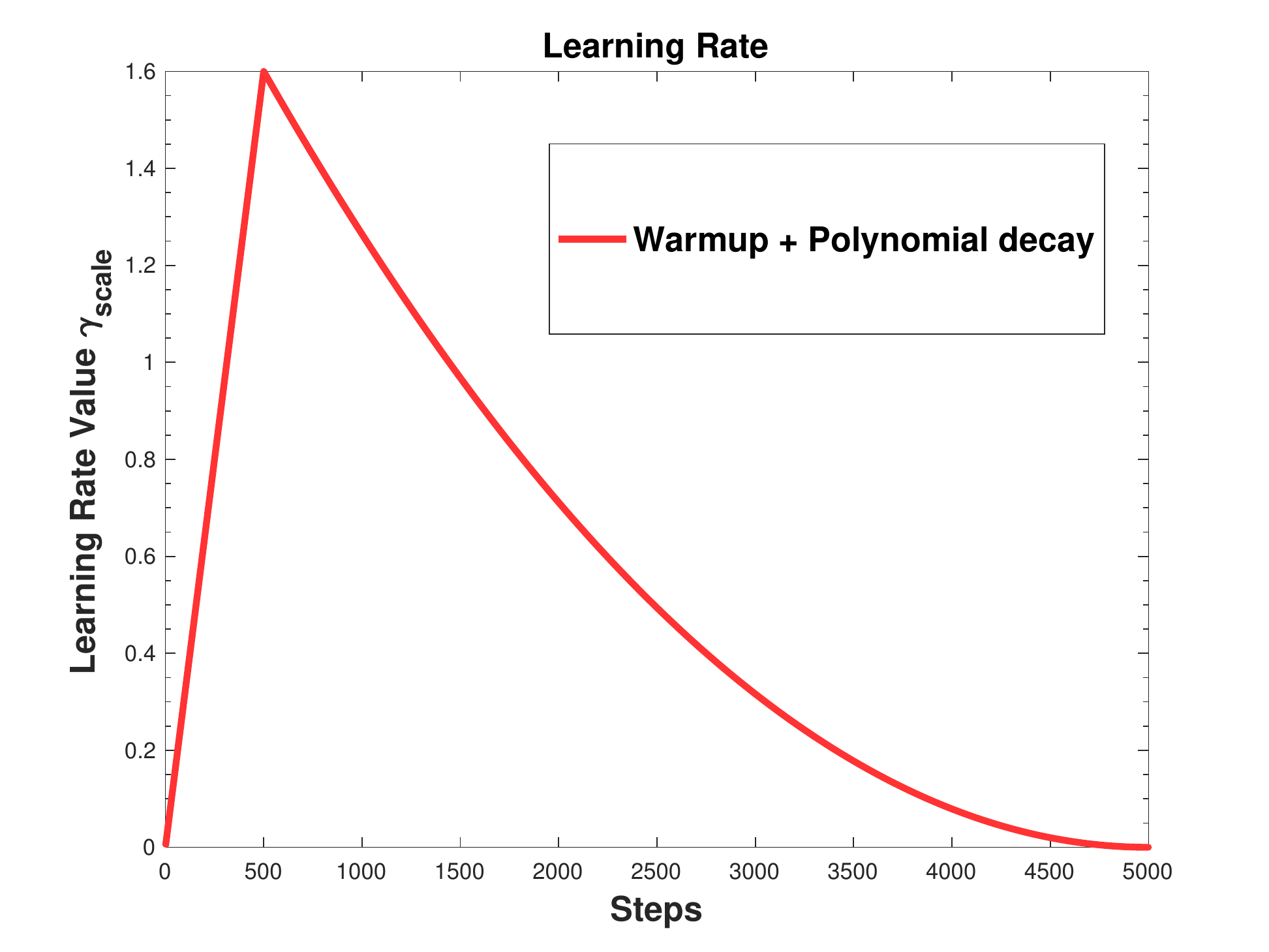}
	\end{subfigure}
	\caption{Training loss and Top-1 testing accuracy of training ResNet56 and VGG11 (with batch normalization layer) on CIFAR-10. Batch size $B$ scales up from $B=128$ to $B=8192$. The right figure presents the variation of $\gamma_{scale}$ in Eq.~(\ref{lars}) when $B=2048$. }
	\label{fig::cifars}
\end{figure*}

\subsection{Why  LARS?}
We test the upper bound of learning rate $\gamma_{k}$ at each layer on three models: $5$-layer FCN, $5$-layer CNN (layer details in the Appendix) and ResNet8 (no batch normalization layer) \cite{he2016deep}.  In the experiments, learning rates  are fixed $\gamma_{k}=0.01$ for all layers except one which is selected from $\{10^{-2},10^{-1},1, 10,10^2, 10^3,10^4,10^5\}$. We optimize models using mNAG with $B=128$ and compare epochs required to achieve the same training loss. Results in Figure \ref{fig::lw} demonstrate that the upper bounds of  learning rates can vary greatly at different layers.  Therefore, it is necessary that each layer has its own learning rate. 

From Theorem \ref{them2}, we know that the upper bound of learning rate $\gamma_{k}$ at each layer is dependent on $\frac{1}{L_k}$.  LARS \cite{you2017scaling} scales the learning rate of each layer adaptively at step $t$ by multiplying $\frac{\|(w_t)_k\|_2}{\|\frac{1}{B} \sum_{i \in I_t}  \nabla_k f_i(w_t) \|_2}$ in Eq.~(\ref{lars}). 
From Assumption \ref{ass_lip}, we can think of  LARS as scaling the learning rate at layer $k$ by multiplying the approximation of $\frac{1}{L_k} \approx  \frac{\|(w_t)_k\|_2}{\left\|  \nabla_k f(w_t)\right\|_2}$, where we make $v_k=0$ and $w_t+U_kv_k=0$.  Therefore, the procedure of LARS is consistent with our theoretical analysis in Theorem \ref{them2} that learning rate of layer $k$ is dependent on the Lipschitz constant at this layer $\gamma_{k} = \frac{\gamma}{L_k}$. We compare LARS with mNAG using a large batch size. Results in Figure \ref{fig::cifars} demonstrate that LARS converges much faster than mNAG  when  $B=8192$. mNAG even diverges in training VGG11 using CIFAR-10. In the experiments, $\gamma_{base}  = 0.1$, $B_{base} = 128$, and $\eta=0.001$ for LARS algorithm.

\subsection{Linear Learning Rate Scaling}
\label{sec:linearscale}
Linear learning rate scaling has been very popular since \cite{goyal2017accurate,krizhevsky2014one,li2017scaling}. However, there is little theoretical understanding of this technique for momentum methods. Based on our analysis in Section \ref{sec:analysis}, we know that the linear learning rate  scaling is from  following  two reasons:

\hspace{0.36cm} {({I})} According to the discussion about Data Parallelism in Section \ref{sec4:dis}, we know that  when $B$ is scaled up by $c$ times to $cB$, the problem can converge to the same error after $\frac{T}{c}$ iterations, as long as $\gamma$ is also scaled up by $B$ times. 

\hspace{0.36cm}  {({II})} According to Theorem \ref{them2}, as long as $\frac{(1-\beta)B}{8L_kM_k}$ dominates the upper bound of the learning rate $\gamma_{k}$ at layer $k$, its upper bound scales linearly with the batch size $B$.

The second case requires that $\frac{B}{M_k}$ to be very small. The layer-wise gradient variance factor $M_k$  is closely related to both model and data. In \cite{shallue2018measuring}, authors find that different models usually have different maximum useful batch size.  The variance factor $M_k$ is highly dependent on the dataset and close to the gradient diversity in \cite{yin2018gradient}. We can draw the same conclusion as \cite{yin2018gradient} that mNAG admits better speedup on problems with higher gradient diversity.

\begin{figure*}[t]
	\centering
	\vspace*{-5pt}
	\begin{subfigure}[b]{0.48\textwidth}
		\centering
		\includegraphics[width=3.5in]{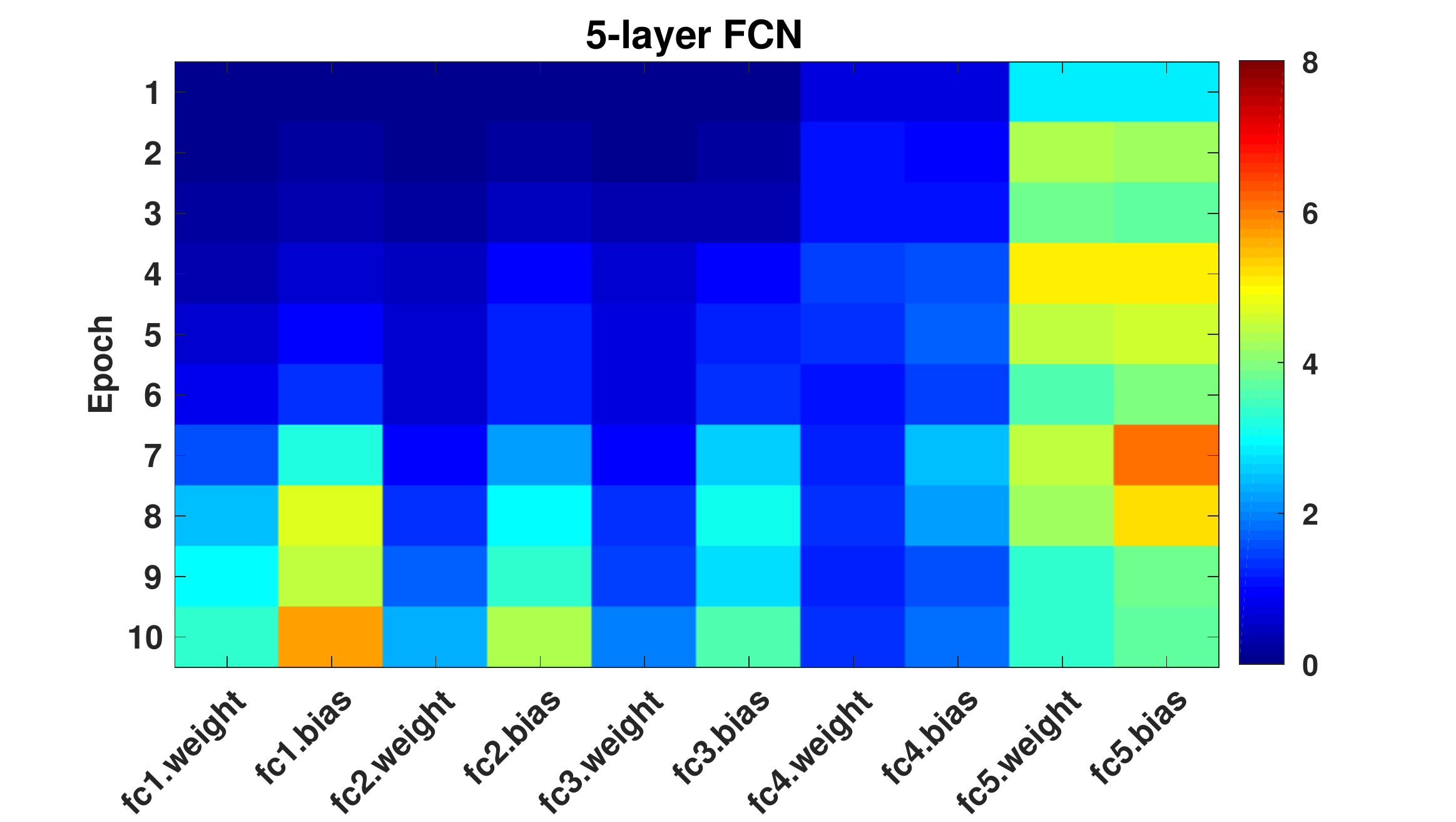}
	\end{subfigure}
	\begin{subfigure}[b]{0.48\textwidth}
		\centering
		\includegraphics[width=3.5in]{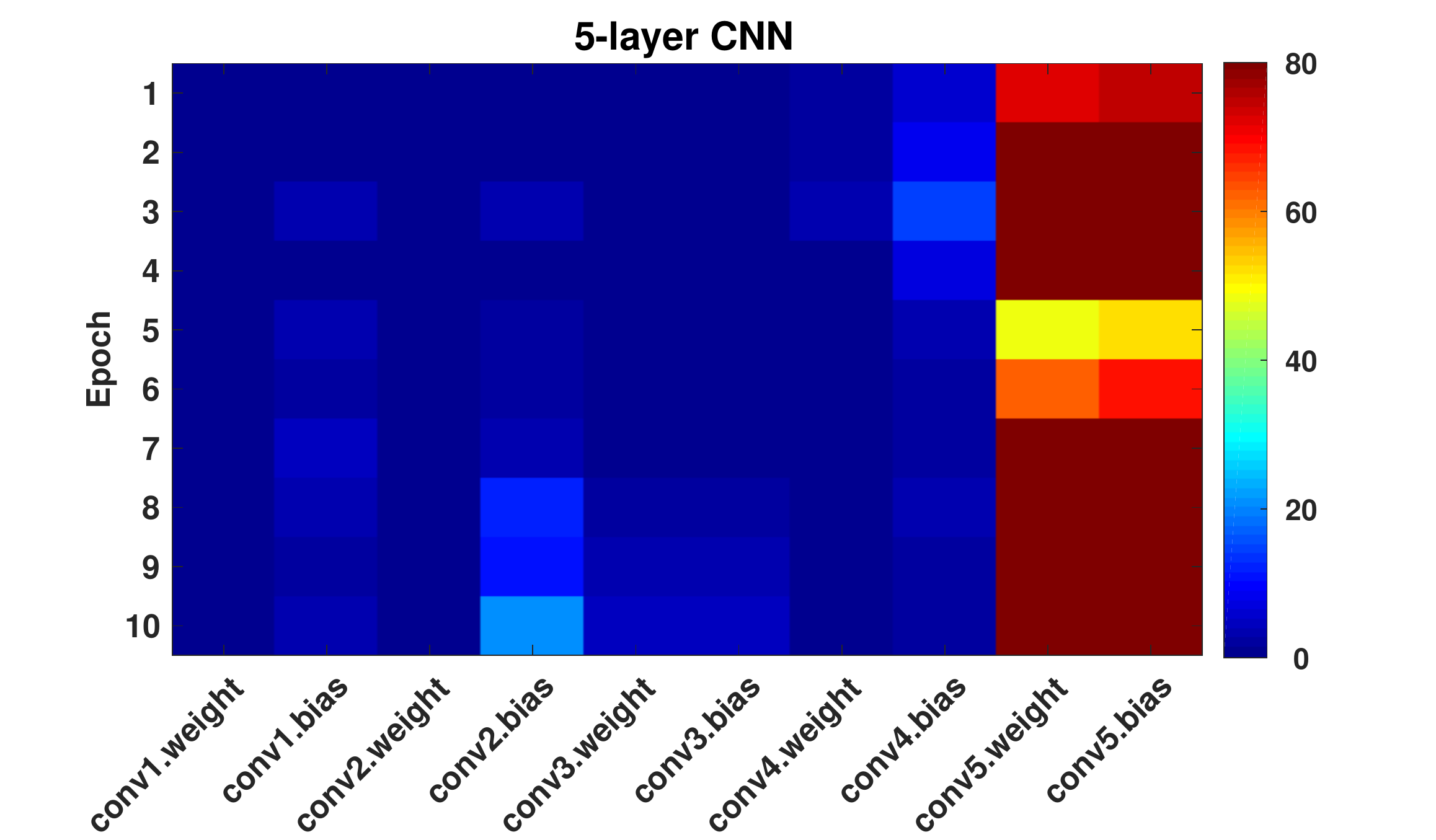}
	\end{subfigure}
	\caption{Variation of $M_k$ for $10$ epochs. We train $5$-layer FCN and $5$-layer CNN with sigmoid activation on  MNIST. }
	\label{fig::wp}
	\vspace*{-10pt}
\end{figure*}

\begin{figure*}[t]
	\centering
	\vspace*{-5pt}
	\begin{subfigure}[b]{0.48\textwidth}
		\centering
		\includegraphics[width=3in]{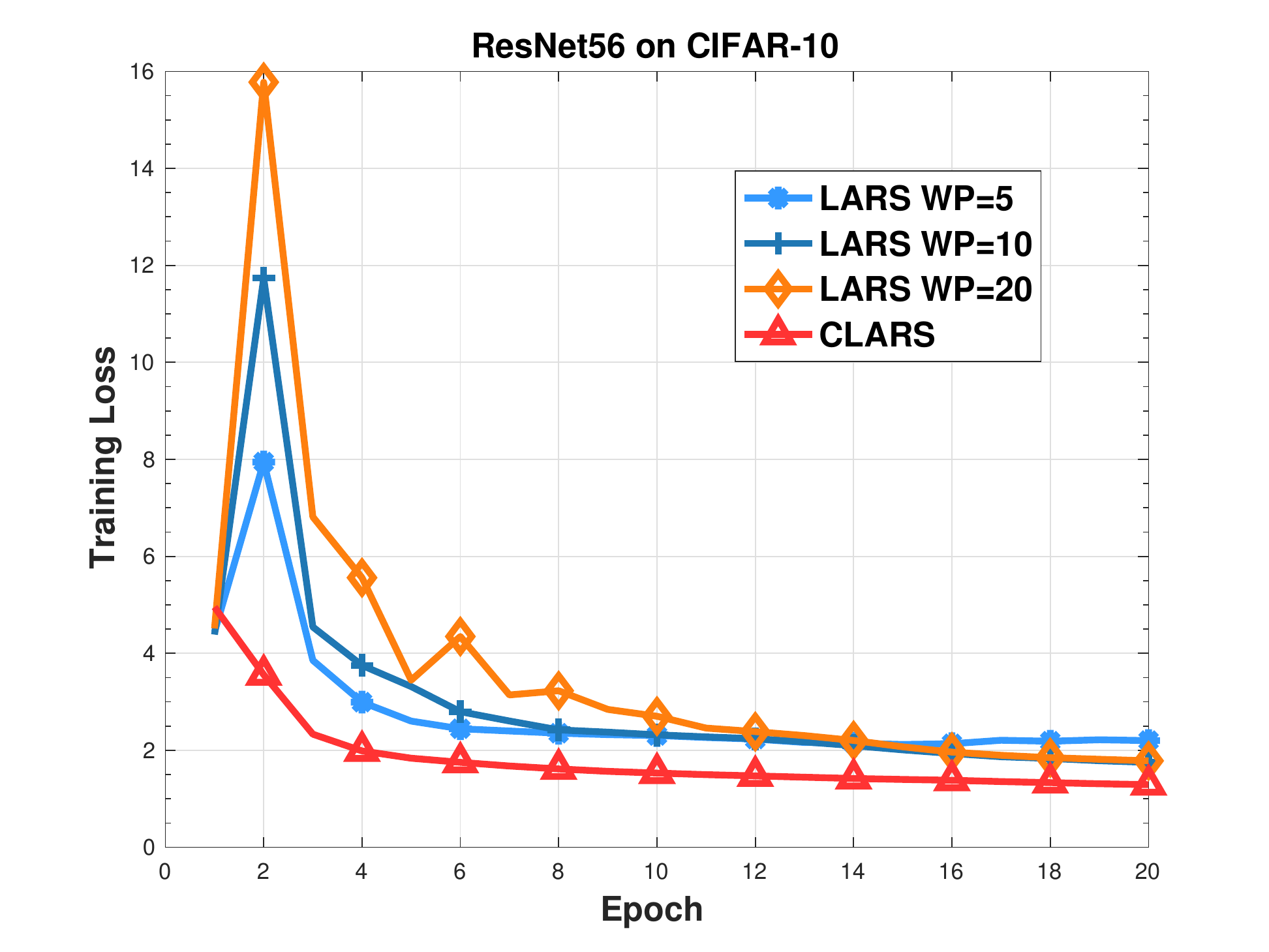}
	\end{subfigure}
	\begin{subfigure}[b]{0.48\textwidth}
		\centering
		\includegraphics[width=3in]{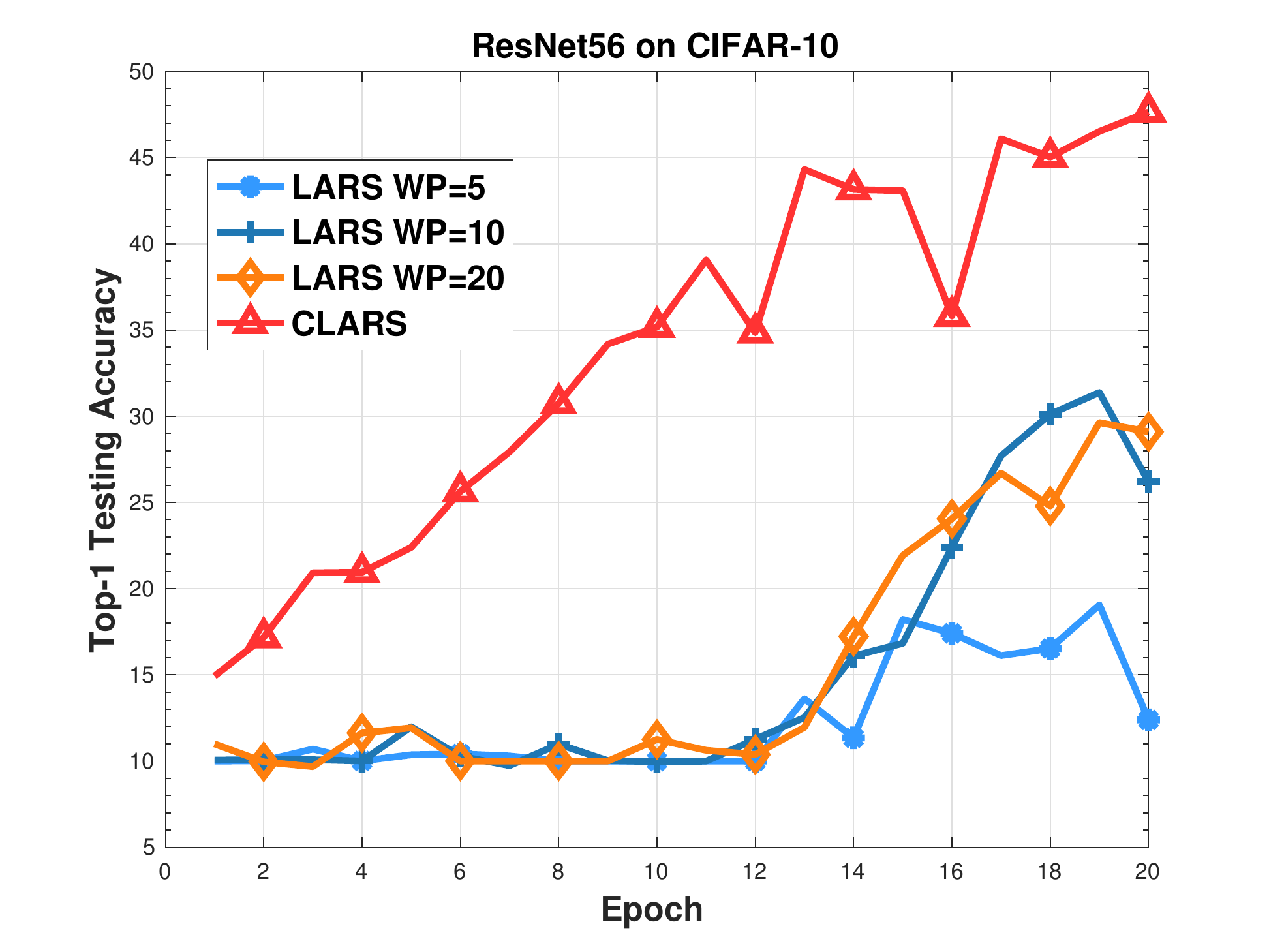}
	\end{subfigure}
	\begin{subfigure}[b]{0.46\textwidth}
		\centering
		\includegraphics[width=3in]{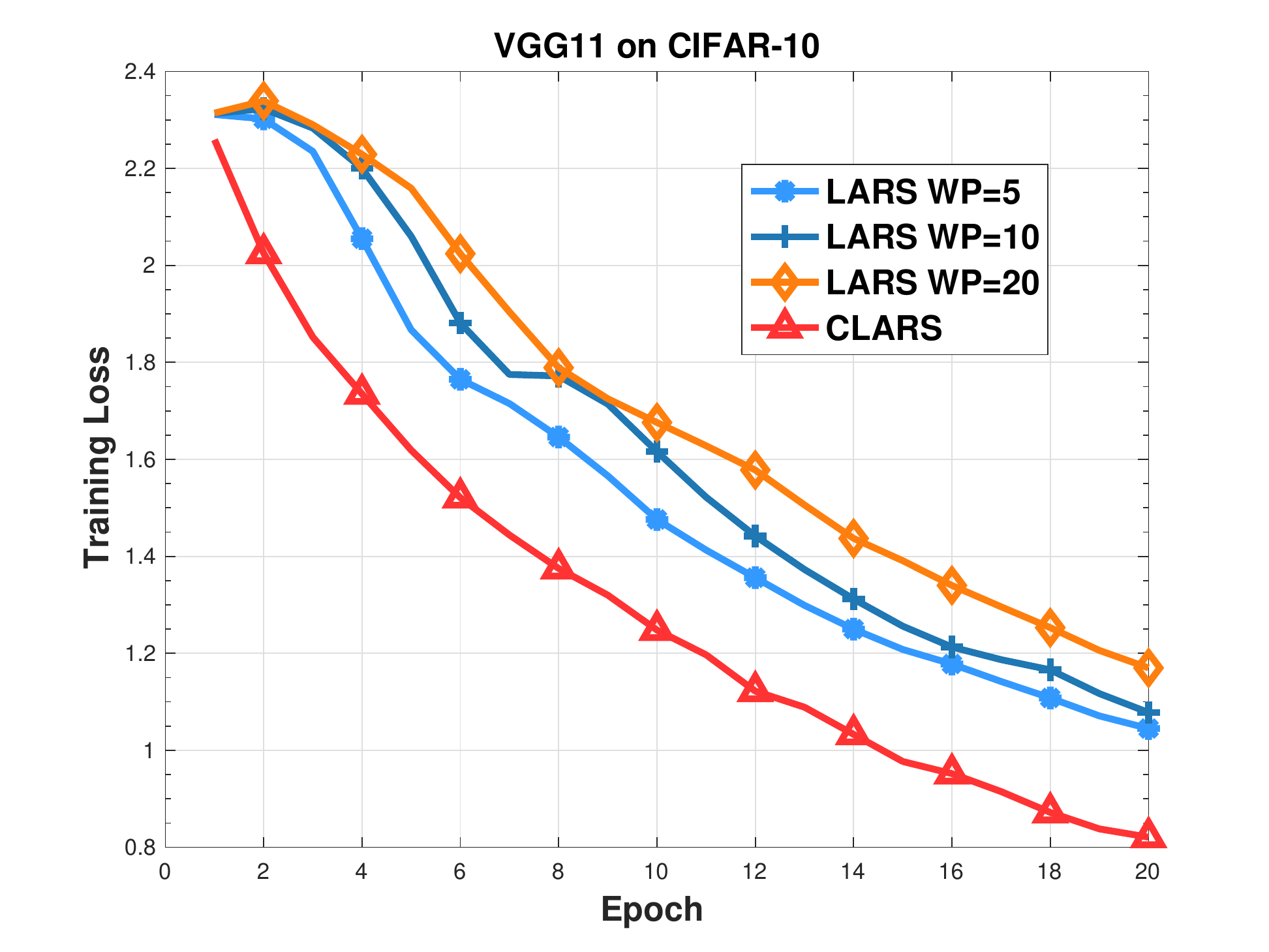}
	\end{subfigure}
	\begin{subfigure}[b]{0.46\textwidth}
		\centering
		\includegraphics[width=3in]{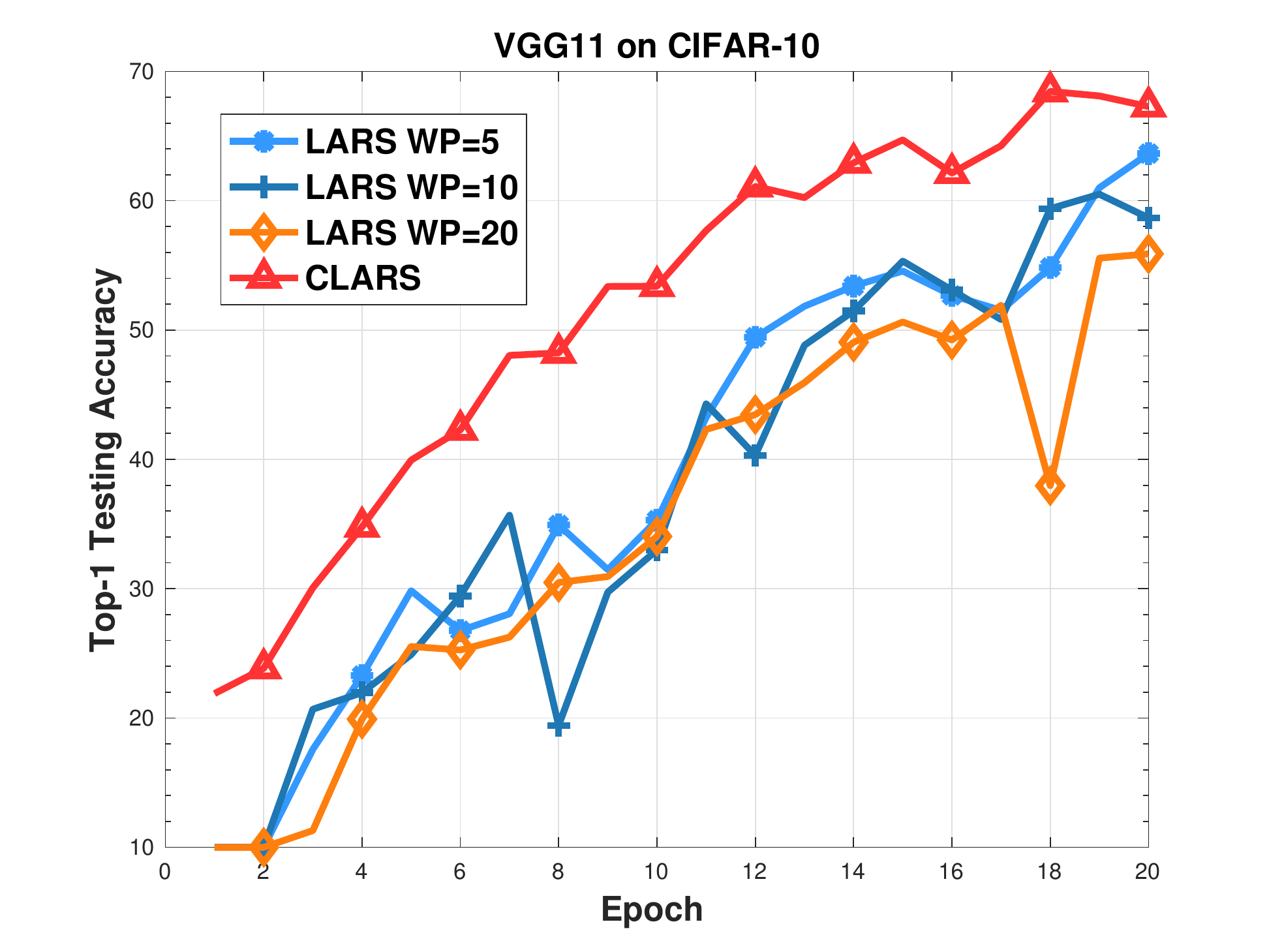}
	\end{subfigure}
	\caption{Comparison between LARS (with gradual warmup) and CLARS algorithm. We optimize ResNet56 and VGG11 on CIFAR-10  with batch size $B=8192$ for $20$ epochs.}
	\label{fig::warmup}
\end{figure*}

In Figure \ref{fig::cifars}, we train ResNet56 \cite{he2016deep} and VGG11 with batch normalization layer \cite{ioffe2015batch,simonyan2014very} on CIFAR-10 \cite{krizhevsky2009learning} for $200$ epochs.  We use LARS optimizer  with gradual warmup ($20$ epochs) and polynomial learning rate decay as \cite{you2017scaling}, which is also visualized in the right side of Figure \ref{fig::cifars}. We scale up the batch size from $128$ to $8192$ and employ the linear learning rate scaling.
Results in Figure \ref{fig::cifars} show that the convergence rates of LARS with batch size from $128$ to $8192$ are similar and the linear speedup is guaranteed when the computations are parallelized on multiple devices. Because the learning rate schedule is tuned for large-batch training, we may observe accuracy improvements when the batch size scales up.

\begin{figure*}[t]
	\begin{subfigure}[b]{0.32\textwidth}
		\centering
		\includegraphics[width=2.36in]{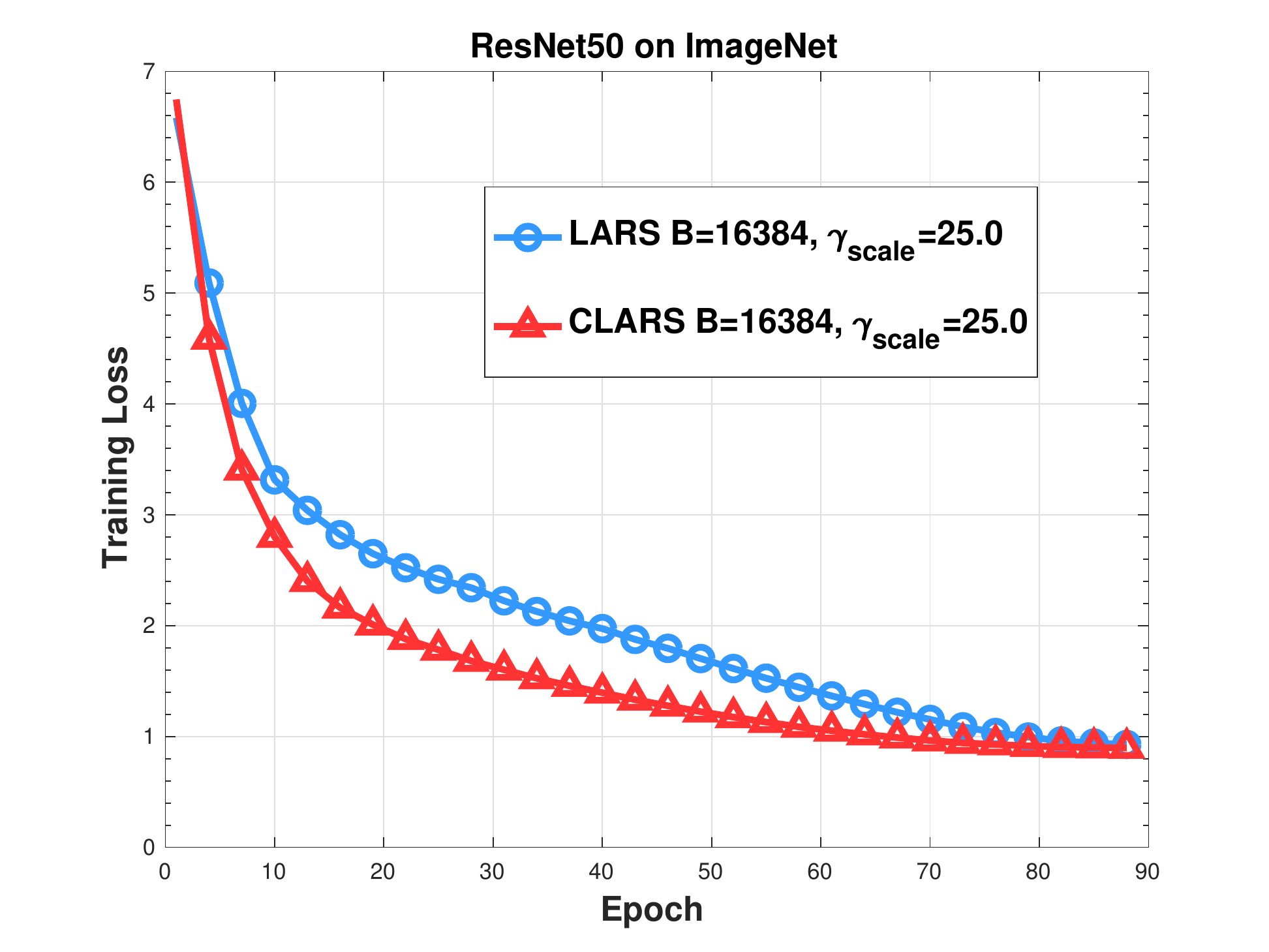}
	\end{subfigure}
	\begin{subfigure}[b]{0.32\textwidth}
		\centering
		\includegraphics[width=2.36in]{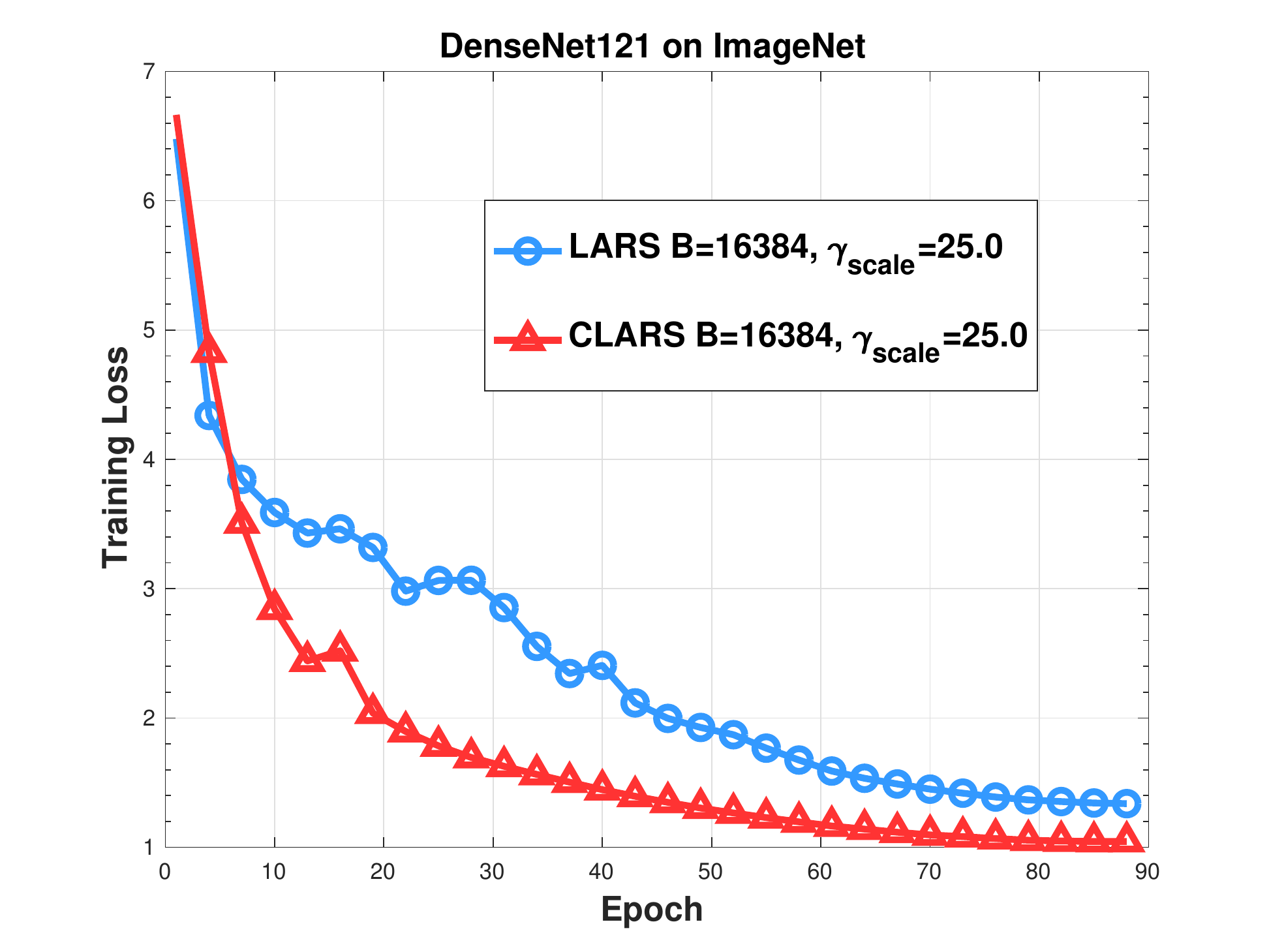}
	\end{subfigure}
	\begin{subfigure}[b]{0.32\textwidth}
		\centering
		\includegraphics[width=2.36in]{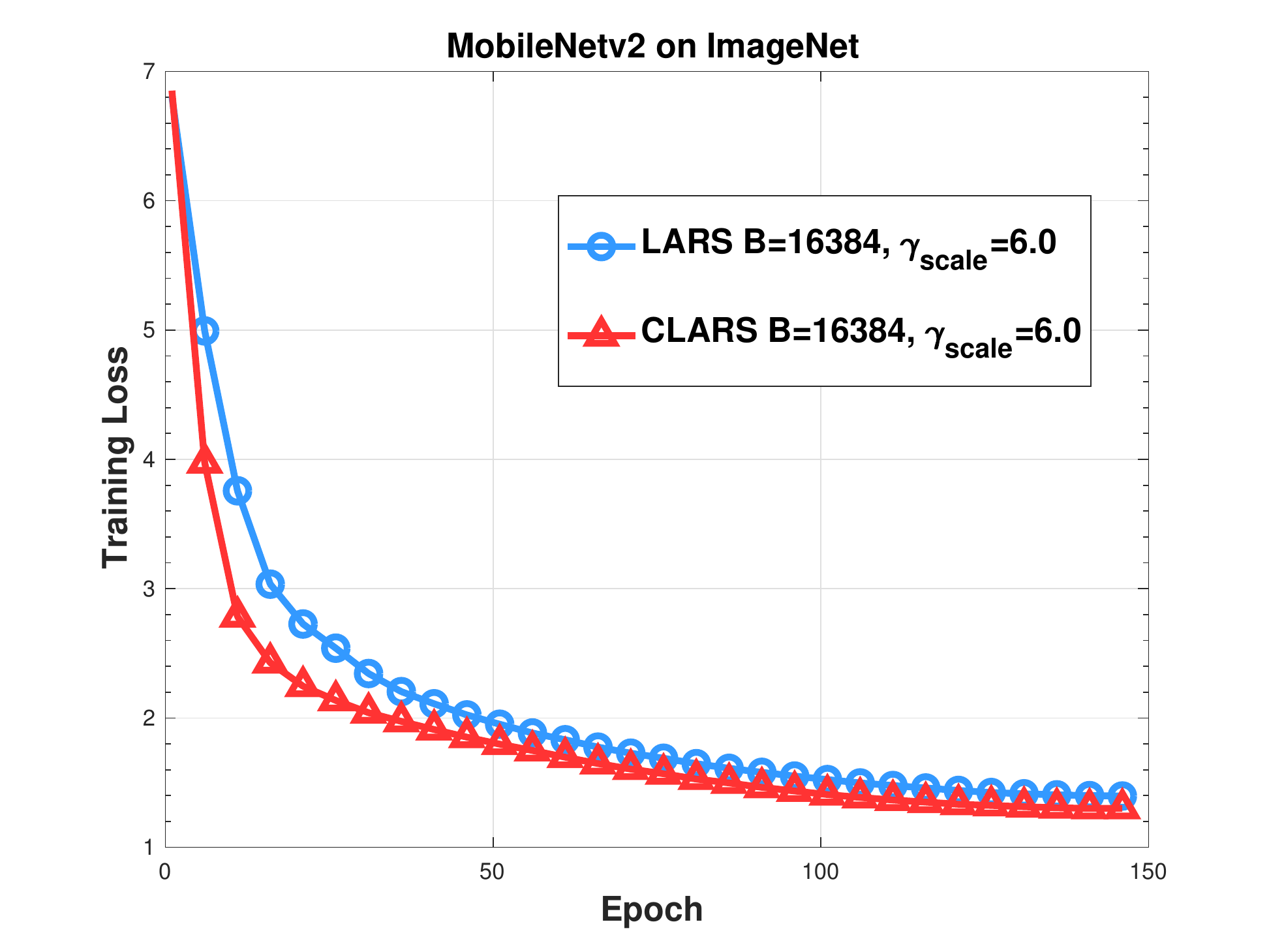}
	\end{subfigure}
	\begin{subfigure}[b]{0.32\textwidth}
		\centering
		\includegraphics[width=2.36in]{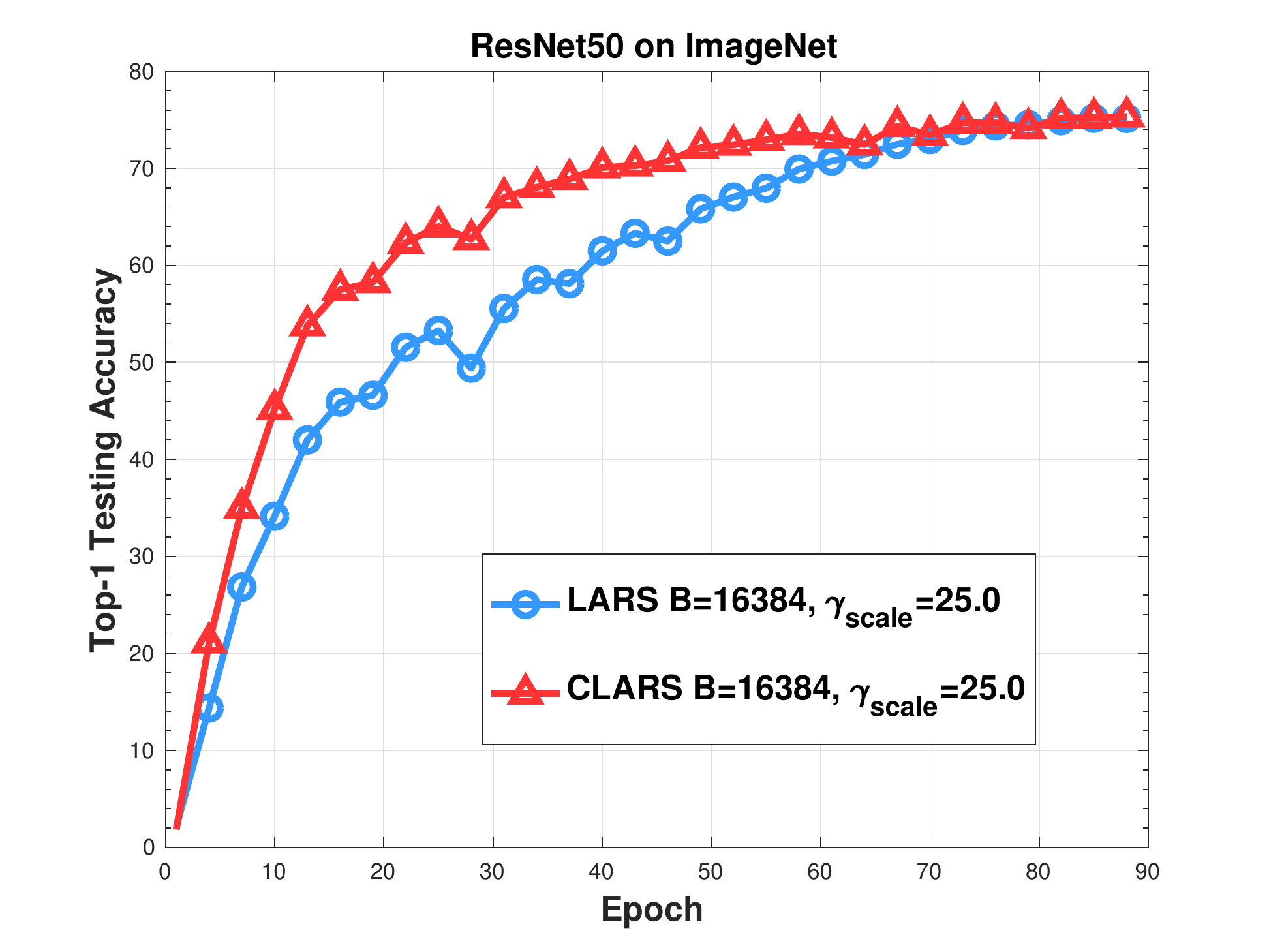}
	\end{subfigure}
	\begin{subfigure}[b]{0.35\textwidth}
		\centering
		\includegraphics[width=2.38in]{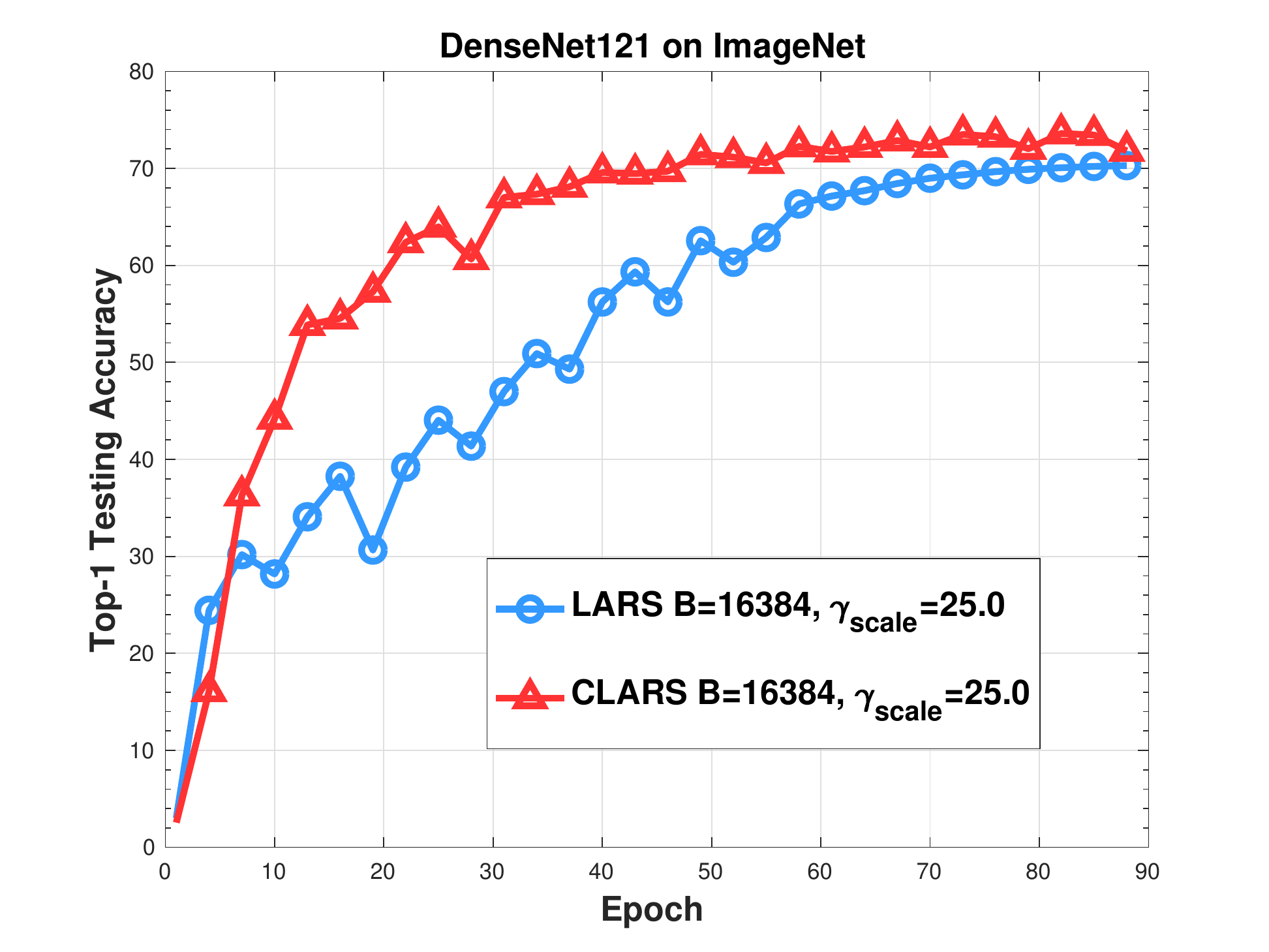}
	\end{subfigure}
	\begin{subfigure}[b]{0.3\textwidth}
		\centering
		\includegraphics[width=2.36in]{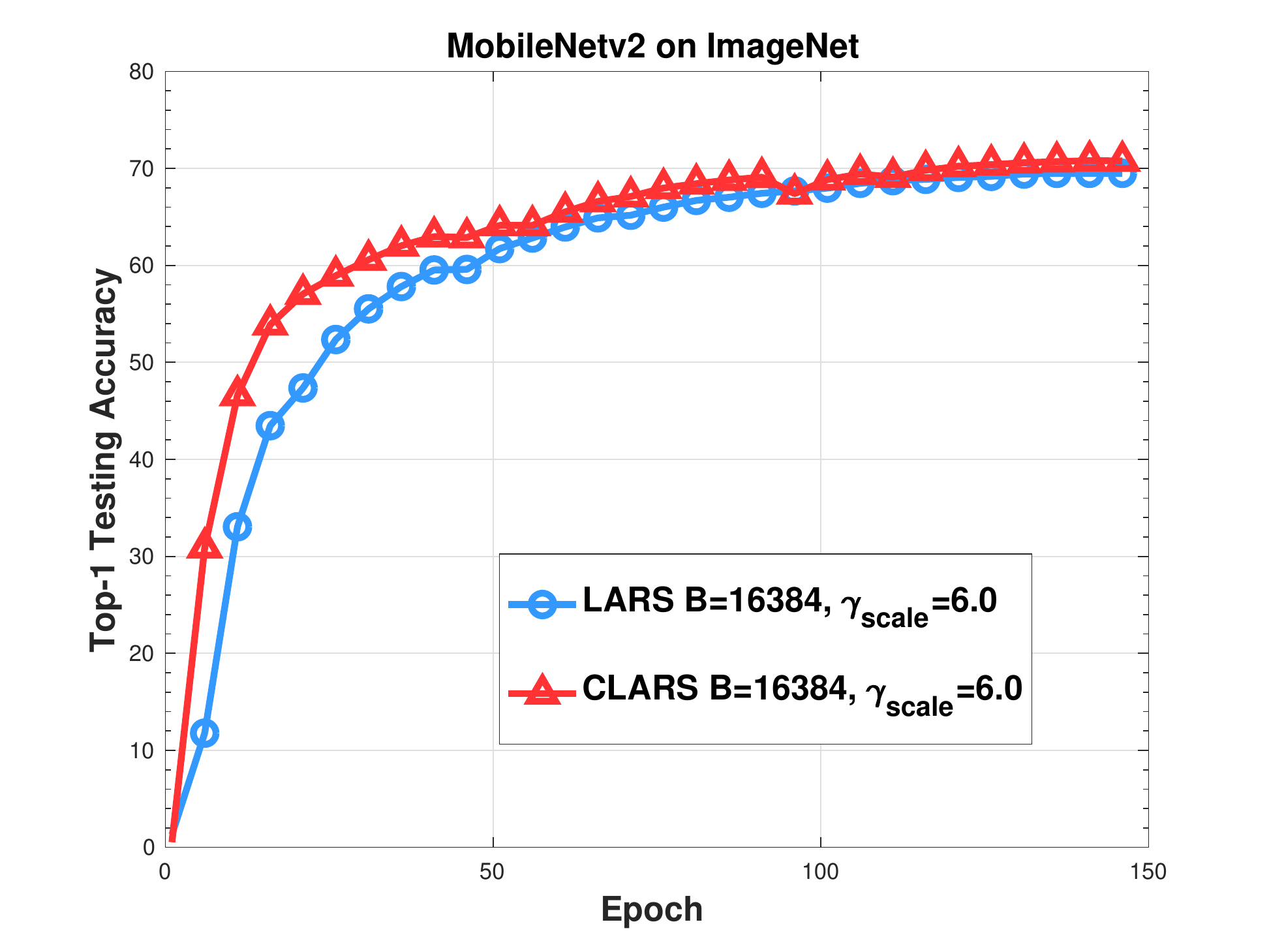}
	\end{subfigure}
	\caption{Comparison between LARS and CLARS 
		on ImageNet.  We train ResNet50, DenseNet121 for  $90$ epochs with batch size $B=16384$ and $\gamma_{scale}=25.0$. MobileNetv2 is trained for $150$ epochs with batch size $B=16384$ and $\gamma_{scale}=6.0$.}
	\label{fig::imagenet}
\end{figure*}

\subsection{One Hypothesis About Warmup}
The gradual warmup was  essential for large-batch deep learning optimization because linearly scaled $\gamma_{scale}$ can be so large that the loss cannot converge in early epochs \cite{goyal2017accurate}.  In the gradual warmup, $\gamma_{scale}$ is replaced with a small value at the beginning and increased back gradually after a few epochs.

According to our analysis in Theorem \ref{them2}, we guess that the gradual warmup is to simulate the function of $\frac{1}{M_k}$ in the upper bound of learning rate. We train $5$-layer FCN, $5$-layer CNN on MNIST \cite{lecun1998gradient} and ResNet8 on CIFAR-10 using mNAG  for $50$ epochs.  Constant learning rate  $0.001$ is used for all layers and  batch size $B=128$. After each epoch, we approximate the gradient variance factor $M_k$ by computing the ratio of  $\frac{1}{n}\sum_{i=1}^n\left\| \nabla_k f_i(w_t) \right\|_2^2$ to $ \|\frac{1}{n}\sum_{i=1}^n \nabla_k f_i(w_t) \|_2^2$ on training data. Figure \ref{fig::wp} presents the variation of $M_k$ at each layers. It is obvious that $M_k$ of top layers are larger than other layers. Thus, smaller learning rates should be used on top layers at early epochs. Our  observation matches the result in \cite{gotmare2018closer} that freezing fully connected layers at early epochs allows for comparable performance with warmup.

\subsection{Warmup is Not Necessary}
\label{sec:exp}
We evaluate the proposed Algorithm \ref{alg_sgd} by conducting extensive experiments.
To reduce the time consumption in computing $M_k$, we approximate it using $M_k \approx \frac{\left\| \frac{1}{B} \sum_{i \in I_t} \nabla_k f_i(w_t) \right\|_2^2}{ \frac{1}{|J_t|} \sum_{j \in J_t} \left\| \nabla_k f_j(w_t) \right\|_2^2}  $, where $|J_t|=512$. The numerator is known after the gradient computation, and the denominator is obtained in a small size.  Since $|J_t|\ll B$, the computational time of approximating  $M_k$ can be ignored when the computation is amortized on multiple devices.  In Figure \ref{fig::warmup}, we make a comparison between LARS (with gradual warmup) and the proposed CLARS algorithm. We train ResNet56 and VGG11 (with batch normalization layer) on CIFAR-10 with batch size $B=8192$ for $20$ epochs.  Standard data preprocessing techniques are used as in \cite{he2016deep}. For LARS with the gradual warmup, we test three warmup epochs $\{5,10,20\}$ and keep $\gamma_{scale}=6.4$ after the warmup. For CLARS, we keep $\gamma_{scale}=6.4$ for $20$ epochs. $\eta$ is tuned from $\{10^{-4},10^{-3},10^{-2}, 10^{-1}\}$ for both methods
Visualization in Figure \ref{fig::warmup} shows that CLARS always outperforms LARS by a large margin. Results demonstrate that warmup is not necessary in large-batch deep learning training and CLARS is a better option for practical implementation.

We also evaluate CLARS algorithm by training ResNet50, DenseNet121, and MobileNetv2 on ImageNet {\cite{deng2009imagenet}}. 
Because there are not enough GPUs to compute $16384$ gradients at one time, we set batch size $B=512$ and accumulate the gradients for $32$ steps before updating the model  as \cite{you2017scaling}. Following the official implementation\footnote[1]{\url{https://github.com/tensorflow/models/blob/master/official/resnet/resnet_run_loop.py}}, we set $\eta=10^{-3}$ for LARS, $\gamma_{scale}=25.0$ for $B=16384$ and adjust the learning rate using $5$-epoch warmup and polynomial decay. For CLARS, there is no warmup and we set $\eta=10^{-2}$ (LARS always diverges with this value). 
Experimental results in Figure \ref{fig::imagenet} present that  CLARS algorithm always converges much faster than the state-of-the-art large-batch optimizer LARS on advanced neural networks. Besides, CLARS  can obtain better test error than LARS. 

\section{Conclusion}
In this paper, we propose a novel Complete Layer-wise Adaptive Rate Scaling (CLARS) algorithm to remove  warmup in the large-batch deep learning training. Then, we introduce fine-grained analysis and prove the convergence of the proposed algorithm for non-convex problems.  Based on our analysis, we bridge the gap between several large-batch deep learning optimization heuristics and theoretical underpins. Extensive experiments demonstrate that the proposed algorithm outperforms  gradual warmup by a large margin and defeats the convergence of the state-of-the-art large-batch optimizer (LARS) in training advanced deep neural networks on ImageNet dataset.

{\small
\bibliographystyle{ieee_fullname}
\bibliography{egbib}

\begin{thebibliography}{10}\itemsep=-1pt

\bibitem{akiba2017extremely}
Takuya Akiba, Shuji Suzuki, and Keisuke Fukuda.
\newblock Extremely large minibatch sgd: training resnet-50 on imagenet in 15
  minutes.
\newblock {\em arXiv preprint arXiv:1711.04325}, 2017.

\bibitem{beck2013convergence}
Amir Beck and Luba Tetruashvili.
\newblock On the convergence of block coordinate descent type methods.
\newblock {\em SIAM journal on Optimization}, 23(4):2037--2060, 2013.

\bibitem{bottou2016optimization}
L{\'e}on Bottou, Frank~E Curtis, and Jorge Nocedal.
\newblock Optimization methods for large-scale machine learning.
\newblock {\em arXiv preprint arXiv:1606.04838}, 2016.

\bibitem{dean2012large}
Jeffrey Dean, Greg Corrado, Rajat Monga, Kai Chen, Matthieu Devin, Mark Mao,
  Andrew Senior, Paul Tucker, Ke Yang, Quoc~V Le, et~al.
\newblock Large scale distributed deep networks.
\newblock In {\em Advances in neural information processing systems}, pages
  1223--1231, 2012.

\bibitem{deng2009imagenet}
Jia Deng, Wei Dong, Richard Socher, Li-Jia Li, Kai Li, and Li Fei-Fei.
\newblock Imagenet: A large-scale hierarchical image database.
\newblock In {\em Computer Vision and Pattern Recognition, 2009. CVPR 2009.
  IEEE Conference on}, pages 248--255. Ieee, 2009.

\bibitem{devlin2018bert}
Jacob Devlin, Ming-Wei Chang, Kenton Lee, and Kristina Toutanova.
\newblock Bert: Pre-training of deep bidirectional transformers for language
  understanding.
\newblock {\em arXiv preprint arXiv:1810.04805}, 2018.

\bibitem{ghadimi2016accelerated}
Saeed Ghadimi and Guanghui Lan.
\newblock Accelerated gradient methods for nonconvex nonlinear and stochastic
  programming.
\newblock {\em Mathematical Programming}, 156(1-2):59--99, 2016.

\bibitem{gotmare2018closer}
Akhilesh Gotmare, Nitish~Shirish Keskar, Caiming Xiong, and Richard Socher.
\newblock A closer look at deep learning heuristics: Learning rate restarts,
  warmup and distillation.
\newblock {\em arXiv preprint arXiv:1810.13243}, 2018.

\bibitem{goyal2017accurate}
Priya Goyal, Piotr Doll{\'a}r, Ross Girshick, Pieter Noordhuis, Lukasz
  Wesolowski, Aapo Kyrola, Andrew Tulloch, Yangqing Jia, and Kaiming He.
\newblock Accurate, large minibatch sgd: Training imagenet in 1 hour.
\newblock {\em arXiv preprint arXiv:1706.02677}, 2017.

\bibitem{he2017mask}
Kaiming He, Georgia Gkioxari, Piotr Doll{\'a}r, and Ross Girshick.
\newblock Mask r-cnn.
\newblock In {\em Proceedings of the IEEE international conference on computer
  vision}, pages 2961--2969, 2017.

\bibitem{he2016deep}
Kaiming He, Xiangyu Zhang, Shaoqing Ren, and Jian Sun.
\newblock Deep residual learning for image recognition.
\newblock In {\em Proceedings of the IEEE conference on computer vision and
  pattern recognition}, pages 770--778, 2016.

\bibitem{hochreiter1997long}
Sepp Hochreiter and J{\"u}rgen Schmidhuber.
\newblock Long short-term memory.
\newblock {\em Neural computation}, 9(8):1735--1780, 1997.

\bibitem{hoffer2017train}
Elad Hoffer, Itay Hubara, and Daniel Soudry.
\newblock Train longer, generalize better: closing the generalization gap in
  large batch training of neural networks.
\newblock In {\em Advances in Neural Information Processing Systems}, pages
  1731--1741, 2017.

\bibitem{huang2016densely}
Gao Huang, Zhuang Liu, Kilian~Q Weinberger, and Laurens van~der Maaten.
\newblock Densely connected convolutional networks.
\newblock {\em arXiv preprint arXiv:1608.06993}, 2016.

\bibitem{ioffe2015batch}
Sergey Ioffe and Christian Szegedy.
\newblock Batch normalization: Accelerating deep network training by reducing
  internal covariate shift.
\newblock {\em arXiv preprint arXiv:1502.03167}, 2015.

\bibitem{jia2018highly}
Xianyan Jia, Shutao Song, Wei He, Yangzihao Wang, Haidong Rong, Feihu Zhou,
  Liqiang Xie, Zhenyu Guo, Yuanzhou Yang, Liwei Yu, et~al.
\newblock Highly scalable deep learning training system with mixed-precision:
  Training imagenet in four minutes.
\newblock {\em arXiv preprint arXiv:1807.11205}, 2018.

\bibitem{keskar2016large}
Nitish~Shirish Keskar, Dheevatsa Mudigere, Jorge Nocedal, Mikhail Smelyanskiy,
  and Ping Tak~Peter Tang.
\newblock On large-batch training for deep learning: Generalization gap and
  sharp minima.
\newblock {\em arXiv preprint arXiv:1609.04836}, 2016.

\bibitem{krizhevsky2014one}
Alex Krizhevsky.
\newblock One weird trick for parallelizing convolutional neural networks.
\newblock {\em arXiv preprint arXiv:1404.5997}, 2014.

\bibitem{krizhevsky2009learning}
Alex Krizhevsky and Geoffrey Hinton.
\newblock Learning multiple layers of features from tiny images.
\newblock 2009.

\bibitem{krizhevsky2012imagenet}
Alex Krizhevsky, Ilya Sutskever, and Geoffrey~E Hinton.
\newblock Imagenet classification with deep convolutional neural networks.
\newblock In {\em Advances in neural information processing systems}, pages
  1097--1105, 2012.

\bibitem{lecun1998gradient}
Yann LeCun, L{\'e}on Bottou, Yoshua Bengio, Patrick Haffner, et~al.
\newblock Gradient-based learning applied to document recognition.
\newblock {\em Proceedings of the IEEE}, 86(11):2278--2324, 1998.

\bibitem{li2017scaling}
Mu Li.
\newblock {\em Scaling Distributed Machine Learning with System and Algorithm
  Co-design}.
\newblock PhD thesis, PhD thesis, Intel, 2017.

\bibitem{lian2015asynchronous}
Xiangru Lian, Yijun Huang, Yuncheng Li, and Ji Liu.
\newblock Asynchronous parallel stochastic gradient for nonconvex optimization.
\newblock In {\em Advances in Neural Information Processing Systems}, pages
  2737--2745, 2015.

\bibitem{lian2016comprehensive}
Xiangru Lian, Huan Zhang, Cho-Jui Hsieh, Yijun Huang, and Ji Liu.
\newblock A comprehensive linear speedup analysis for asynchronous stochastic
  parallel optimization from zeroth-order to first-order.
\newblock In {\em Advances in Neural Information Processing Systems}, pages
  3054--3062, 2016.

\bibitem{mahajan2018exploring}
Dhruv Mahajan, Ross Girshick, Vignesh Ramanathan, Kaiming He, Manohar Paluri,
  Yixuan Li, Ashwin Bharambe, and Laurens van~der Maaten.
\newblock Exploring the limits of weakly supervised pretraining.
\newblock In {\em Proceedings of the European Conference on Computer Vision
  (ECCV)}, pages 181--196, 2018.

\bibitem{mikami2018imagenet}
Hiroaki Mikami, Hisahiro Suganuma, Yoshiki Tanaka, Yuichi Kageyama, et~al.
\newblock Imagenet/resnet-50 training in 224 seconds.
\newblock {\em arXiv preprint arXiv:1811.05233}, 2018.

\bibitem{mnih2013playing}
Volodymyr Mnih, Koray Kavukcuoglu, David Silver, Alex Graves, Ioannis
  Antonoglou, Daan Wierstra, and Martin Riedmiller.
\newblock Playing atari with deep reinforcement learning.
\newblock {\em arXiv preprint arXiv:1312.5602}, 2013.

\bibitem{nesterov1983method}
Yurii~E Nesterov.
\newblock A method for solving the convex programming problem with convergence
  rate o (1/k\^{} 2).
\newblock In {\em Dokl. akad. nauk Sssr}, volume 269, pages 543--547, 1983.

\bibitem{paszke2017automatic}
Adam Paszke, Sam Gross, Soumith Chintala, Gregory Chanan, Edward Yang, Zachary
  DeVito, Zeming Lin, Alban Desmaison, Luca Antiga, and Adam Lerer.
\newblock Automatic differentiation in pytorch.
\newblock In {\em NIPS-W}, 2017.

\bibitem{radford2019language}
Alec Radford, Jeffrey Wu, Rewon Child, David Luan, Dario Amodei, and Ilya
  Sutskever.
\newblock Language models are unsupervised multitask learners.
\newblock {\em OpenAI Blog}, 1:8, 2019.

\bibitem{ren2015faster}
Shaoqing Ren, Kaiming He, Ross Girshick, and Jian Sun.
\newblock Faster r-cnn: Towards real-time object detection with region proposal
  networks.
\newblock In {\em Advances in neural information processing systems}, pages
  91--99, 2015.

\bibitem{santurkar2018does}
Shibani Santurkar, Dimitris Tsipras, Andrew Ilyas, and Aleksander Madry.
\newblock How does batch normalization help optimization?
\newblock In {\em Advances in Neural Information Processing Systems}, pages
  2483--2493, 2018.

\bibitem{shallue2018measuring}
Christopher~J Shallue, Jaehoon Lee, Joe Antognini, Jascha Sohl-Dickstein, Roy
  Frostig, and George~E Dahl.
\newblock Measuring the effects of data parallelism on neural network training.
\newblock {\em arXiv preprint arXiv:1811.03600}, 2018.

\bibitem{silver2017mastering}
David Silver, Julian Schrittwieser, Karen Simonyan, Ioannis Antonoglou, Aja
  Huang, Arthur Guez, Thomas Hubert, Lucas Baker, Matthew Lai, Adrian Bolton,
  et~al.
\newblock Mastering the game of go without human knowledge.
\newblock {\em Nature}, 550(7676):354, 2017.

\bibitem{simonyan2014very}
Karen Simonyan and Andrew Zisserman.
\newblock Very deep convolutional networks for large-scale image recognition.
\newblock {\em arXiv preprint arXiv:1409.1556}, 2014.

\bibitem{vaswani2017attention}
Ashish Vaswani, Noam Shazeer, Niki Parmar, Jakob Uszkoreit, Llion Jones,
  Aidan~N Gomez, {\L}ukasz Kaiser, and Illia Polosukhin.
\newblock Attention is all you need.
\newblock In {\em Advances in neural information processing systems}, pages
  5998--6008, 2017.

\bibitem{yadan2013multi}
Omry Yadan, Keith Adams, Yaniv Taigman, and Marc'Aurelio Ranzato.
\newblock Multi-gpu training of convnets.
\newblock {\em arXiv preprint arXiv:1312.5853}, 2013.

\bibitem{yang2016unified}
Tianbao Yang, Qihang Lin, and Zhe Li.
\newblock Unified convergence analysis of stochastic momentum methods for
  convex and non-convex optimization.
\newblock {\em arXiv preprint arXiv:1604.03257}, 2016.

\bibitem{yin2018gradient}
Dong Yin, Ashwin Pananjady, Max Lam, Dimitris Papailiopoulos, Kannan
  Ramchandran, and Peter Bartlett.
\newblock Gradient diversity: a key ingredient for scalable distributed
  learning.
\newblock In {\em International Conference on Artificial Intelligence and
  Statistics}, pages 1998--2007, 2018.

\bibitem{ying2018image}
Chris Ying, Sameer Kumar, Dehao Chen, Tao Wang, and Youlong Cheng.
\newblock Image classification at supercomputer scale.
\newblock {\em arXiv preprint arXiv:1811.06992}, 2018.

\bibitem{you2017scaling}
Yang You, Igor Gitman, and Boris Ginsburg.
\newblock Scaling sgd batch size to 32k for imagenet training.
\newblock {\em arXiv preprint arXiv:1708.03888}, 6, 2017.

\bibitem{you2019large}
Yang You, Jonathan Hseu, Chris Ying, James Demmel, Kurt Keutzer, and Cho-Jui
  Hsieh.
\newblock Large-batch training for lstm and beyond.
\newblock {\em arXiv preprint arXiv:1901.08256}, 2019.

\bibitem{zou2018lipschitz}
Dongmian Zou, Radu Balan, and Maneesh Singh.
\newblock On lipschitz bounds of general convolutional neural networks.
\newblock {\em arXiv preprint arXiv:1808.01415}, 2018.

\end{thebibliography}
}

\onecolumn

\appendix

\section{Fine-Grained Analysis of Mini-Batch Gradient Descent}
\begin{lemma}
	Under Assumptions \ref{ass_lip} and \ref{ass_bd}, after applying Eq.~(\ref{newnag}) with $\beta=0$ for $K$ micro-steps from $s=0$ to $s=K-1$, we have the upper bound of loss $\mathbb{E}[f(w_{t+1})]$ as follows:
	{\small
		\begin{eqnarray}
		\mathbb{E} \left[  f(w_{t+1}) \right]
		&\leq & \mathbb{E} \left[ f(w_{t}) \right] - \sum\limits_{k=1}^{K} \frac{\gamma_{k}}{2} \left( 1 -  {L_{k} \gamma_{k}}-  \frac{L_{k} \gamma_{k} M_{k} }{B} -  \frac{L_g^2 \gamma_k M_k  \sum\limits_{k=1}^K\gamma_{k}  }{KB} - \frac{ L_g^2 \gamma_k   \sum\limits_{k=1}^K\gamma_{k}  }{K} \right)\mathbb{E}\left\|   \nabla_{k} f(w_t)  \right\|_2^2  \nonumber \\
		&& + \sum\limits_{k=1}^K  \frac{L_{k} \gamma_{k}^2M}{2B}  +\frac{\sum\limits_{k=1}^K  \gamma_{k} L_g^2M}{2KB} \sum\limits_{k=1}^K \gamma_{k}^2.
		\label{iiq_lem1}
		\end{eqnarray}}
	\label{lem_1}
\end{lemma}
\begin{proof}
	Suppose $K$ layers are updated sequentially from $s=0$ to $K-1$, and we have $w_{t} = w_{t:0}$ and $w_{t+1}= w_{t:K}$. At micro-step $t$:$s$, we set $k(s)=s \pmod{K}+1$. According to Assumption \ref{ass_lip}, we have:
	\begin{eqnarray}
	\mathbb{E} \left[  f(w_{t:s+1}) \right] 
	&\leq &\mathbb{E} \left[ f(w_{t:s})\right] - \mathbb{E}\left< \nabla_{k(s)} f(w_{t:s}),  {\gamma_{k(s)}}   \nabla_{k(s)} f(w_{t}) \right>  + \frac{L_{k(s)} \gamma_{k(s)}^2}{2} \mathbb{E}\left\|  \frac{1}{B} \sum\limits_{i \in I_t}  \nabla_{k(s)} f_{i}(w_t) \right\|_2^2 \nonumber \\
	&= &\mathbb{E} \left[ f(w_{t:s})\right]+ \frac{L_{k(s)} \gamma_{k(s)}^2 }{2} \underbrace{\mathbb{E}\left\| \frac{1}{B} \sum\limits_{i \in I_t}  \nabla_{k(s)} f_{i}(w_t)  \right\|_2^2}_{C_1} - \frac{\gamma_{k(s)}}{2} \bigg( \mathbb{E}\left\|\nabla_{k(s)} f(w_{t:s}) \right\|_2^2 \nonumber \\
	&&  + \mathbb{E} \left\|\nabla_{k(s)} f(w_{t}) \right\|_2^2 -\underbrace{ \mathbb{E}\left\|\nabla_{k(s)} f(w_{t:s}) - \nabla_{k(s)} f(w_{t}) \right\|_2^2}_{C_2}   \bigg).
	\label{iq_01}
	\end{eqnarray}
	In the following context, we will prove that $C_1$ and $C_2$ are upper bounded. At first, we can get the upper bound of $C_1$ as follows:
	\begin{eqnarray}
	C_1 &=& \mathbb{E}\left\| \frac{1}{B} \sum\limits_{i \in I_t} \left( \nabla_{k(s)} f_{i}(w_t) - \nabla_{k(s)} f(w_t) + \nabla_{k(s)} f(w_t)  \right)\right\|_2^2 \nonumber \\
	&= & \frac{1}{B^2} \mathbb{E}\left\| \sum\limits_{i \in I_t} \left( \nabla_{k(s)} f_{i}(w_t) - \nabla_{k(s)} f(w_t) \right)   \right\|_2^2 + \mathbb{E}\left\|   \nabla_{k(s)} f(w_t)  \right\|_2^2 \nonumber \\
	&\leq  & \frac{1}{B^2} \sum\limits_{i \in I_t}  \mathbb{E}\left\|  \nabla_{k(s)} f_{i}(w_t) - \nabla_{k(s)} f(w_t)  \right\|_2^2 + \mathbb{E}\left\|   \nabla_{k(s)} f(w_t)  \right\|_2^2 \nonumber \\
	&\leq & \left(1+ \frac{M_{k(s)}}{B}\right) \mathbb{E}\left\|   \nabla_{k(s)} f(w_t)  \right\|_2^2 + \frac{M}{B},
	\label{iq_02}
	\end{eqnarray}
	where the second equality follows from $\mathbb{E} \left< \nabla_{k(s)} f_{i}(w_t) - \nabla_{k(s)} f(w_t), \nabla_{k(s)} f(w_t) \right> = 0$ and the first inequality follows from  $ \mathbb{E} \left\| \sum\limits_{i=1}^n \xi_i \right\|_2^2\leq \sum\limits_{i=1}^n  \mathbb{E} \left\|  \xi_i  \right\|_2^2$ if $\mathbb{E}[\xi_i]=0$  and the second inequality follows from Assumption \ref{ass_bd}.
	Following ``global'' Lipschitz continuous in  Assumption \ref{ass_linear}, we can bound  $C_2$ as follows:
	\begin{eqnarray}
	C_2 & \leq & \frac{ L_g^2}{K}  \mathbb{E} \left\|  w_{t:s} - w_t   \right\|_2^2 \nonumber \\
	&= &    \frac{ L_g^2}{K}  \mathbb{E} \left\|\sum\limits_{j =0 }^{s-1}   \frac{\gamma_{k(j)}}{B} \sum\limits_{i \in I_t} \nabla_{k(j)} f_{i}(w_t)  \right\|_2^2 \nonumber \\
	&= & \frac{L_{g}^2 }{KB^2}  \sum\limits_{j=0}^{s-1}   \gamma_{k(j)}^2  \mathbb{E} \left\| \sum\limits_{i \in I_t} \left(   \nabla_{k(j)} f_{i}(w_t)  -   \nabla_{k(j)} f(w_t) \right) + B  \nabla_{k(j)} f(w_t)   \right\|_2^2 \nonumber \\
	&\leq &  \frac{L_{g}^2 }{KB^2}  \sum\limits_{j=0}^{s-1}   \gamma_{k(j)}^2  \sum\limits_{i \in I_t} \mathbb{E} \left\|   \nabla_{k(j)} f_{i}(w_t)  -   \nabla_{k(j)} f(w_t)    \right\|_2^2 +  \frac{L_{g}^2 }{K}  \sum\limits_{j=0}^{s-1}   \gamma_{k(j)}^2  \mathbb{E} \left\|  \nabla_{k(j)} f(w_t)   \right\|_2^2 \nonumber \\
	&\leq &  \frac{L_{g}^2 }{KB}  \sum\limits_{j=0}^{s-1}   \gamma_{k(j)}^2   \left( M_{k(j)} \mathbb{E} \left\|  \nabla_{k(j)} f(w_t)   \right\|_2^2 + M\right) +  \frac{L_{g}^2 }{K}   \sum\limits_{j=0}^{s-1}   \gamma_{k(j)}^2  \mathbb{E} \left\|  \nabla_{k(j)} f(w_t)   \right\|_2^2 \nonumber \\
	&\leq &  \frac{L_{g}^2 }{KB}  \sum\limits_{k=1}^{K}   \gamma_{k}^2   \left( M_{k} \mathbb{E} \left\|  \nabla_{k} f(w_t)   \right\|_2^2 + M\right) + \frac{L_{g}^2 }{K}  \sum\limits_{k=1}^{K}   \gamma_{k}^2  \mathbb{E} \left\|  \nabla_{k} f(w_t)   \right\|_2^2, 
	\label{iq_03}
	\end{eqnarray}
	where the first inequality follows from Assumption \ref{ass_lip}, the second inequality follows from  $ \mathbb{E} \left\| \sum\limits_{i=1}^n \xi_i \right\|_2^2\leq \sum\limits_{i=1}^n  \mathbb{E} \left\|  \xi_i  \right\|_2^2$ if $\mathbb{E}[\xi_i]=0$,
	the third inequality follows from Assumption \ref{ass_bd} and the last inequality is because $s\leq K-1$.
	Combing inequalities (\ref{iq_01}), (\ref{iq_02}) and (\ref{iq_03}), we have:
	\begin{eqnarray}
	\mathbb{E} \left[  f(w_{t:s+1}) \right]  &\leq & \mathbb{E} \left[ f(w_{t:s}) \right]  - \bigg( \frac{\gamma_{k(s)}}{2} -  \frac{L_{k(s)} \gamma_{k(s)}^2}{2} -  \frac{L_{k(s)} \gamma_{k(s)}^2 M_{k(s)} }{2B} \bigg)\mathbb{E}\left\|   \nabla_{k(s)} f(w_t)  \right\|_2^2  \nonumber \\
	&&+ \frac{\gamma_{k(s)} L_g^2}{2KB} \sum\limits_{k=1}^K \gamma_{k}^2 M_k \mathbb{E}\|\nabla_k f(w_t) \|_2^2   +   \frac{\gamma_{k(s)} L_g^2}{2K} \sum\limits_{k=1}^K \gamma_{k}^2\mathbb{E}\|\nabla_k f(w_t) \|_2^2 \nonumber  \\
	&&   + \frac{L_{k(s)} \gamma_{k(s)}^2M}{2B}  + \frac{\gamma_{k(s)} L_g^2M}{2KB} \sum\limits_{k=1}^K \gamma_{k}^2.
	\end{eqnarray}
	By summing from $s=0$ to $K-1$, because  $w_{t} = w_{t:0}$ and $w_{t+1}= w_{t:K}$, we have:
	\begin{eqnarray}
	&&		\mathbb{E} \left[  f(w_{t+1}) \right]  \nonumber \\
	&\leq & \mathbb{E} \left[ f(w_{t}) \right] - \sum\limits_{k=1}^{K} \left( \frac{\gamma_{k}}{2} -  \frac{L_{k} \gamma_{k}^2}{2} -  \frac{L_{k} \gamma_{k}^2 M_{k} }{2B} \right)\mathbb{E}\left\|   \nabla_{k} f(w_t)  \right\|_2^2  \nonumber \\
	&& + \frac{\sum\limits_{k=1}^{K}  \gamma_{k} L_g^2}{2KB} \sum\limits_{k=1}^K \gamma_{k}^2 M_k \mathbb{E}\|\nabla_k f(w_t) \|_2^2  +  \frac{\sum\limits_{k=1}^K  \gamma_{k} L_g^2}{2K} \sum\limits_{k=1}^K \gamma_{k}^2\mathbb{E}\|\nabla_k f(w_t) \|_2^2 \nonumber \\
	&& + \sum\limits_{k=1}^K  \frac{L_{k} \gamma_{k}^2M}{2B}  + \frac{\sum\limits_{k=1}^K  \gamma_{k} L_g^2M}{2KB} \sum\limits_{k=1}^K \gamma_{k}^2 \nonumber \\ 
	&\leq & \mathbb{E} \left[ f(w_{t}) \right] - \sum\limits_{k=1}^{K} \frac{\gamma_{k}}{2} \left( 1 -  {L_{k} \gamma_{k}}-  \frac{L_{k} \gamma_{k} M_{k} }{B} -  \frac{L_g^2 \gamma_k M_k  \sum\limits_{k=1}^K\gamma_{k}  }{KB} - \frac{ L_g^2 \gamma_k   \sum\limits_{k=1}^K\gamma_{k}  }{K} \right)\mathbb{E}\left\|   \nabla_{k} f(w_t)  \right\|_2^2  \nonumber \\
	&& + \sum\limits_{k=1}^K  \frac{L_{k} \gamma_{k}^2M}{2B}  +\frac{\sum\limits_{k=1}^K  \gamma_{k} L_g^2M}{2KB} \sum\limits_{k=1}^K \gamma_{k}^2.
	\end{eqnarray}
	$\hfill\blacksquare$
\end{proof}

\textbf{Proof of Theorem \ref{them1}}
\begin{proof}
	Following Lemma \ref{lem_1} and defining $\kappa_k = \frac{L_g}{L_k} \leq \kappa $, if $\gamma_{k}$ satisfies following inequalities:
	\begin{eqnarray}
	{L_{k} \gamma_{k}} & \leq & \frac{1}{8}, \\
	\frac{L_{k} \gamma_{k}M_k}{B} & \leq & \frac{1}{8}, \\
	\frac{L_g^2 \gamma_k M_k  \sum\limits_{k=1}^K\gamma_{k}  }{KB}  &\leq & \frac{1}{4}, \\
	\frac{L_g^2 \gamma_k   \sum\limits_{k=1}^K\gamma_{k}}{K}  & \leq & \frac{1}{4},
	\end{eqnarray}
	which  are equivalent to  $\gamma_k \leq \min \left\{\frac{1}{8L_k}, \frac{B}{8 L_kM_k} \right\} $ and $ \frac{1}{K} \sum\limits_{k=1}^K\gamma_k \leq \min \left\{\frac{1}{2L_g},   \frac{1}{2 L_g }  \sqrt{\frac{B}{M_g  }} \right\} $.   Therefore, it holds that:
	\begin{eqnarray}
	\mathbb{E} \left[  f(w_{t+1}) \right] &\leq & \mathbb{E} \left[ f(w_{t}) \right]  -   \sum\limits_{k=1}^K  \frac{ \gamma_k }{8}\mathbb{E}\left\|   \nabla_k f(w_t)  \right\|_2^2 + \sum\limits_{k=1}^K  \frac{(2 + \kappa )ML_k \gamma_k^2}{4B} .
	\end{eqnarray}
	Rearranging the above inequality  and summing it from $t=0$ to $T-1$, we have:
	\begin{eqnarray}
	\frac{1}{8} \sum\limits_{t=0}^{T-1}  \sum\limits_{k=1}^K \gamma_k \mathbb{E}\left\|   \nabla_k f(w_t)  \right\|_2^2  &\leq &   f(w_{0})  -   \mathbb{E} \left[  f(w_{T}) \right]  +  \frac{(2 + \kappa )MT}{4B}   \sum\limits_{k=1}^K L_k \gamma_k^2.
	\end{eqnarray}
	Because  $f(w_T) \geq f_{\inf}$, let $\gamma_k = \frac{\gamma}{L_k}$  and dividing both sides by $\frac{T}{8} { \sum\limits_{k=1}^K \gamma_k }$, we have:
	\begin{eqnarray}
	\frac{1}{T}  \sum\limits_{t=0}^{T-1} \sum\limits_{k=1}^K q_k \mathbb{E} \left\|   \nabla_k f(w_t)  \right\|_2^2  &\leq & \frac{8(f(w_0)-f_{\inf} )}{T\gamma \sum\limits_{k=1}^K \frac{1}{L_k} } + \frac{( 4+ 2\kappa )M\gamma}{B},
	\end{eqnarray}
	where $q_k = \frac{\frac{1}{L_k}}{\sum\limits_{k=1}^K \frac{1}{L_k}}$.
	We complete the proof.
	$\hfill\blacksquare$
\end{proof}

\textbf{Proof of Corollary \ref{cor_1_1}}
\begin{proof}
	Because $\min\limits_{t \in \{0,...,T-1\}}  \sum\limits_{k=1}^K q_k\mathbb{E}  \left\|   \nabla_k f(w_t)  \right\|_2^2 \leq \frac{1}{T}  \sum\limits_{t=0}^{T-1}   \sum\limits_{k=1}^K q_k\mathbb{E} \left\|   \nabla_k f(w_t)  \right\|_2^2  $ suppose $\frac{1}{8L_k}$ dominates the upper bound of $\gamma_k$, if:
	\begin{eqnarray}
	\gamma & = & \min \left\{ \frac{1}{8}, \sqrt{ \frac{B (f(w_0)-f_{\inf} )  }{TM \sum\limits_{k=1}^K \frac{1}{L_k} }  } \right\}
	\end{eqnarray}
	we have:
	\begin{eqnarray}
	\frac{1}{T}  \sum\limits_{t=0}^{T-1} \sum\limits_{k=1}^K q_k \mathbb{E} \left\|   \nabla_k f(w_t)  \right\|_2^2  &\leq &  \frac{8(f(w_0)-f_{\inf} )}{T \sum\limits_{k=1}^K \frac{1}{L_k} } \max\left\{{8}, \sqrt{ \frac{TM \sum\limits_{k=1}^K \frac{1}{L_k} }{B (f(w_0)-f_{\inf} )  }  }  \right\}  + \frac{( 4+2 \kappa)M }{B} \sqrt{ \frac{B (f(w_0)-f_{\inf} )  }{TM \sum\limits_{k=1}^K \frac{1}{L_k} }  } \nonumber \\
	&\leq &  \frac{64(f(w_0)-f_{\inf} )}{T \sum\limits_{k=1}^K \frac{1}{L_k} } +  \left( 12 + 2\kappa  \right) \sqrt{ \frac{M (f(w_0)-f_{\inf} )  }{TB \sum\limits_{k=1}^K \frac{1}{L_k}  }  }. 
	\end{eqnarray}
	$\hfill\blacksquare$
\end{proof}

\section{Fine-Grained Analysis of Mini-Batch Nesterov's Accelerated Gradient}

Following \cite{yang2016unified}, we  define $
p_t = \frac{\beta }{1-\beta} \left(w_t - w_{t-1} + g_{t-1} \right)$,
where  $w_{-1} = w_0$, $g_t =\sum_{k=1}^K  \gamma_k  U_k   \left( \frac{1}{B} \sum_{i \in I_t}  \nabla_k f_i(w_t)\right) $ and $g_{-1} = 0$.
Let $z_t = w_t + p_t$, we prove that $\mathbb{E} [f(z_{t+1} )]$ is upper bounded at each step in the following Lemma.
\begin{lemma}
	Under Assumptions \ref{ass_lip} and \ref{ass_bd}, after applying Eq.~(\ref{newnag}) for $K$ micro-steps from $s=0$ to $s=K-1$, we have the upper bound of loss $\mathbb{E}[f(z_{t+1})]$ as follows:
	\begin{eqnarray}
	\mathbb{E} [f(z_{t+1} )]  &\leq & \mathbb{E} [f(z_{t} )]  - \sum\limits_{k=1}^K \frac{\gamma_{k}}{2(1-\beta)}    \bigg( 1- \frac{L_{k} \gamma_{k} }{1-\beta} - \frac{L_{k} \gamma_{k}M_{k} }{(1-\beta)B}  \nonumber \\
	&& - \frac{2L_{g}^2 \gamma_{k}    M_{k}  }{(1-\beta)^2  KB}  \sum\limits_{k=1}^{K}   \gamma_{k}  -   \frac{2L_{g}^2 \gamma_{k} }{(1-\beta)^2K}  \sum\limits_{k=1}^{K}   \gamma_{k}    \bigg) \mathbb{E} \left\|\nabla_{k} f(w_t)  \right\|_2^2  \nonumber \\
	&&  +  \frac{ M L_{g}^2 }{(1-\beta)^3  KB}  \sum\limits_{k=1}^{K}  \gamma_{k}   \sum\limits_{k=1}^{K}   \gamma_{k}^2 +  \sum\limits_{k=1}^{K}  \frac{L_{k} \gamma_{k}^2 M}{2(1-\beta)^2 B} +    \sum\limits_{k=1}^{K}  \frac{ L_{g}^2 \gamma_{k}}{(1-\beta)K}  \mathbb{E} \left\|p_{t}  \right\|_2^2.
	\end{eqnarray}
	\label{lem2_1}
\end{lemma}
\begin{proof}
	We define $w_{t:0} = w_t$, and at step $t$:$s$, we have layer index $k(s)=s+1$. Thus, we can rewrite Eq.~(\ref{newnag}) for any $s \in \{0,1,...,K-1\}$ as follows:
	\begin{eqnarray}
	w_{t:s+1} & = &   w_{t:s} - \frac{\gamma_{k(s)}}{B} \sum\limits_{i \in I_t}   U_{k(s)} \nabla_{k(s)} f_{i}(w_t)  + \beta \bigg(w_{t:s} - \frac{\gamma_{k(s)}}{B}  \sum\limits_{i \in I_t} U_{k(s)} \nabla_{k(s)} f_{i}(w_t) \nonumber \\
	&& - w_{t:s-1} + \frac{\gamma_{k(s-1)}}{B}   \sum\limits_{i \in I_t} U_{k(s-1)} \nabla_{k(s-1)} f_{i}(w_t)  \bigg),
	\label{iq_4001}
	\end{eqnarray}
	where we let $w_{t:0} = w_{t:-1}$ and $ \nabla_{k(-1)} f_{i(-1)} (w_{t})= 0$. We also define $p_{t:s}$ as follows:
	\begin{eqnarray}
	p_{t:s} &=&
	\frac{\beta }{1-\beta} \left(w_{t:s} - w_{t:s-1} + \frac{\gamma_{k(s-1)}}{B}  \sum\limits_{i \in I_t} U_{k(s-1)} \nabla_{k(s-1)} f_{i}(w_t) \right).
	\label{iq_4002}
	\end{eqnarray}
	Combining (\ref{iq_4001}) and (\ref{iq_4002}),  we have:
	\begin{eqnarray}
	w_{t:s+1} + p_{t:s+1}  &=&w_{t:s} + p_{t:s} - \frac{\gamma_{k(s)}}{(1-\beta)B}  \sum\limits_{i \in I_t}  U_{k(s)} \nabla_{k(s)} f_{i}(w_t) .
	\end{eqnarray}
	Let $z_{t:s} = w_{t:s} + p_{t:s}$, according to Assumption \ref{ass_lip}, we have:
	\begin{eqnarray}
	\mathbb{E} [f(z_{t:s+1} )]
	&\leq &  \mathbb{E} [f(z_{t:s} )] -  \frac{\gamma_{k(s)}}{(1-\beta)} \mathbb{E}\left<\nabla_{k(s)} f(z_{t:s}),   \nabla_{k(s)} f(w_t)  \right> + \frac{L_{k(s)} \gamma_{k(s)}^2 }{2(1-\beta)^2 }  \mathbb{E} \left\| \frac{1}{B}   \sum\limits_{i \in I_t} \nabla_{k(s)} f_{i}(w_t)   \right\|_2^2 \nonumber \\
	&= &  \mathbb{E} [f(z_{t:s} )] + \frac{L_{k(s)} \gamma_{k(s)}^2 }{2(1-\beta)^2 } \underbrace{ \mathbb{E} \left\|  \frac{1}{B} \sum\limits_{i \in I_t}  \nabla_{k(s)} f_{i}(w_t)   \right\|_2^2}_{C_3}    -\frac{\gamma_{k(s)}}{2(1-\beta)}  \bigg(  \mathbb{E} \left\|\nabla_{k(s)} f(z_{t:s}) \right\|_2^2 +  \mathbb{E} \left\|\nabla_{k(s)} f(w_t)  \right\|_2^2 \nonumber \\
	&& -\underbrace{ \mathbb{E} \left\|\nabla_{k(s)} f(z_{t:s}) - \nabla_{k(s)} f(w_t)  \right\|_2^2}_{C_4}  \bigg).
	\label{iq_5001}
	\end{eqnarray}
	From (\ref{iq_02}), it is easy to know that the upper bound of $C_3$ as follows:
	\begin{eqnarray}
	C_3	&\leq & \left(1+ \frac{M_{k(s)}}{B}\right) \mathbb{E}\left\|   \nabla_{k(s)} f(w_t)  \right\|_2^2 + \frac{M}{B}.
	\label{iq_5002}
	\end{eqnarray}
	We then obtain the upper bound of $C_4$:
	\begin{eqnarray}
	C_4 &\leq  &  \frac{L_{g}^2}{K} \mathbb{E} \left\| z_{t:s} - w_{t:0}  \right\|_2^2 \nonumber \\
	&= & \frac{L_{g}^2}{K} \mathbb{E} \left\|z_{t:s} - z_{t:0} +  z_{t:0} - w_{t:0}  \right\|_2^2 \nonumber \\
	&\leq &  \frac{ 2L_{g}^2 }{(1-\beta)^2KB^2} \mathbb{E} \left\| \sum\limits_{j=0}^{s-1} \gamma_{k(j)}
	\sum\limits_{i \in I_t } \nabla_{k(j)} f_i (w_t)   \right\|_2^2 + \frac{2L_{g}^2}{K}  \mathbb{E} \left\|p_{t:0}  \right\|_2^2\nonumber \\
	&\leq&  \frac{2L_{g}^2 }{(1-\beta)^2 K B}  \sum\limits_{k=1}^{K}   \gamma_{k}^2   \left( M_{k} \mathbb{E} \left\|  \nabla_{k} f(w_t)   \right\|_2^2 + M\right) +  \frac{2L_{g}^2 }{(1-\beta)^2K}  \sum\limits_{k=1}^{K}   \gamma_{k}^2  \mathbb{E} \left\|  \nabla_{k} f(w_t)   \right\|_2^2 + \frac{2L_{g}^2}{K}  \mathbb{E} \left\|p_{t:0}  \right\|_2^2,
	\label{iq_5003}
	\end{eqnarray}
	where the first inequality follows from $\|a+b\|^2_2 \leq 2\|a\|_2^2+ 2\|b\|_2^2 $ and the second inequality follows from inequality (\ref{iq_03}).
	After combining (\ref{iq_5001}), (\ref{iq_5002}) and (\ref{iq_5003}), we have:
	\begin{eqnarray}
	\mathbb{E} [f(z_{t:s+1} )] &\leq & \mathbb{E} [f(z_{t:s} )]  - \bigg( \frac{\gamma_{k(s)}}{2(1-\beta)}  - \frac{L_{k(s)} \gamma_{k(s)}^2 }{2(1-\beta)^2} - \frac{L_{k(s)} \gamma_{k(s)}^2M_{k(s)} }{2(1-\beta)^2B}  \bigg) \mathbb{E} \left\|\nabla_{k(s)} f(w_t)  \right\|_2^2   \nonumber \\
	&& +  \frac{L_{g}^2 \gamma_{k(s)} }{(1-\beta)^3  KB}  \sum\limits_{k=1}^{K}   \gamma_{k}^2    M_{k} \mathbb{E} \left\|  \nabla_{k} f(w_t)   \right\|_2^2 +  \frac{L_{g}^2 \gamma_{k(s)} }{(1-\beta)^3K}  \sum\limits_{k=1}^{K}   \gamma_{k}^2  \mathbb{E} \left\|  \nabla_{k} f(w_t)   \right\|_2^2  \nonumber \\
	&&  +  \frac{ M L_{g}^2 \gamma_{k(s)} }{(1-\beta)^3  KB}  \sum\limits_{k=1}^{K}   \gamma_{k}^2 + \frac{L_{k(s)} \gamma_{k(s)}^2 M}{2(1-\beta)^2 B} + \frac{ L_{g}^2 \gamma_{k(s)}}{(1-\beta)K}  \mathbb{E} \left\|p_{t}  \right\|_2^2.
	\label{iq_5004}
	\end{eqnarray}
	Summing  (\ref{iq_5004}) from $s=0$ to $K-1$, because $z_{t:0} = z_{t}$ and $z_{t:K} = z_{t+1}$, we have:
	\begin{eqnarray}
	\mathbb{E} [f(z_{t+1} )]  &\leq & \mathbb{E} [f(z_{t} )]  - \sum\limits_{k=1}^K \frac{\gamma_{k}}{2(1-\beta)}    \bigg( 1- \frac{L_{k} \gamma_{k} }{1-\beta} - \frac{L_{k} \gamma_{k}M_{k} }{(1-\beta)B}  \nonumber \\
	&& - \frac{2L_{g}^2 \gamma_{k}    M_{k}  }{(1-\beta)^2  KB}  \sum\limits_{k=1}^{K}   \gamma_{k}  -   \frac{2L_{g}^2 \gamma_{k} }{(1-\beta)^2K}  \sum\limits_{k=1}^{K}   \gamma_{k}    \bigg) \mathbb{E} \left\|\nabla_{k} f(w_t)  \right\|_2^2  \nonumber \\
	&&  +  \frac{ M L_{g}^2 }{(1-\beta)^3  KB}  \sum\limits_{k=1}^{K}  \gamma_{k}   \sum\limits_{k=1}^{K}   \gamma_{k}^2 +  \sum\limits_{k=1}^{K}  \frac{L_{k} \gamma_{k}^2 M}{2(1-\beta)^2 B} +    \sum\limits_{k=1}^{K}  \frac{ L_{g}^2 \gamma_{k}}{(1-\beta)K}  \mathbb{E} \left\|p_{t}  \right\|_2^2.
	\label{iq_5005}
	\end{eqnarray}
	$\hfill\blacksquare$
\end{proof}

\begin{lemma}
	Under Assumptions \ref{ass_lip} and \ref{ass_bd}, after applying Eq.~(\ref{nag}) from $t=0$ to $T-1$, the following inequality is satisfied that:
	\begin{eqnarray}
	\sum\limits_{t=0}^{T-1}  \mathbb{E} \left\|p_{t}  \right\|_2^2
	&\leq &  \sum\limits_{k=1}^K \frac{\beta^4\gamma_{k}^2MT}{(1-\beta)^4 B }   +  \sum\limits_{k=1}^K  \left( 1+ \frac{M_k}{B}\right) \frac{\beta^4\gamma_{k}^2}{(1-\beta)^4 } \sum\limits_{t=0}^{T-1}  \mathbb{E} \left\|   \nabla_k f(w_{t}) \right\|_2^2.
	\end{eqnarray}	
	\label{lem2_2}
\end{lemma}
\begin{proof}
	We  define $p_t$ as follows:
	\begin{eqnarray}
	p_t &=& \frac{\beta }{1-\beta} \left(w_t - w_{t-1} + g_{t-1} \right),
	\end{eqnarray}
	where we let $w_{-1} = w_0$, $g_t = \sum\limits_{k=1}^K  \frac{\gamma_k}{B} \sum\limits_{i \in I_t}  \nabla_k f_i(w_t) $ and $g_{-1} = 0$.  According to the update of mini-batch NAG in Eq.~(\ref{nag}),
	it holds that:
	\begin{eqnarray}
	w_{t+1} & = &   w_{t} - g_t  + \beta \left(w_t 	- g_t - w_{t-1} + g_{t-1}  \right).
	\end{eqnarray}
	According to the definition of $p_t$, we have:
	\begin{eqnarray}
	p_{t+1} &=& \beta p_t - \frac{\beta^2}{1-\beta} g_t.
	\label{iq_p_func}
	\end{eqnarray}
	According to Eq.~(\ref{iq_p_func}) and $p_{0} = 0$, we know that:
	\begin{eqnarray}
	p_{t}  &= & \beta p_{t-1} - \frac{\beta^2}{1-\beta} g_{t-1} \nonumber \\
	& = & - \frac{\beta^2}{1-\beta} \sum\limits_{j=0}^{t-1}  \beta^{t-1-j} g_j \nonumber \\
	&= & - \frac{\beta^2}{1-\beta} \sum\limits_{j=0}^{t-1}  \beta^{j} g_{t-1-j}.
	\end{eqnarray}
	Let $\Gamma_{t-1} = \sum\limits_{j=0}^{t-1} \beta^j$, we have:
	\begin{eqnarray}
	\mathbb{E} \| p_t \|_2^2 &= & \frac{\beta^4}{(1-\beta)^2}  	\mathbb{E} \left\| \sum\limits_{j=0}^{t-1}  \beta^{j} g_{t-1-j}  \right\|_2^2 \nonumber \\
	& =  &  \frac{\beta^4\Gamma_{t-1}^2}{(1-\beta)^2 }  	\mathbb{E} \left\| \sum\limits_{j=0}^{t-1}  \frac{\beta^{j}}{\Gamma_{t-1}} g_{t-1-j}  \right\|_2^2 \nonumber \\
	&\leq & \frac{\beta^4\Gamma_{t-1}^2}{(1-\beta)^2 }   \sum\limits_{j=0}^{t-1}  \frac{\beta^{j}}{\Gamma_{t-1}} 	\mathbb{E} \left\|g_{t-1-j}   \right\|_2^2 \nonumber \\
	&= & \frac{\beta^4\Gamma_{t-1}}{(1-\beta)^2 }   \sum\limits_{j=0}^{t-1} {\beta^{j}}	\mathbb{E} \left\|g_{t-1-j}  \right\|_2^2,
	\label{iq_3001}
	\end{eqnarray}
	where the inequality is from the convexity of $\|\|_2^2$.  We can get the upper bound of $\mathbb{E}  \left\|g_{t}  \right\|_2^2$ as follows:
	\begin{eqnarray}
	\mathbb{E} \left\| g_{t}  \right\|_2^2 & = & \mathbb{E} \left\| \sum\limits_{k=1}^K  \frac{\gamma_k }{B} \sum\limits_{i \in I_t}  \nabla_k f_i(w_t) \right\|_2^2 \nonumber \\
	& = &  \sum\limits_{k=1}^K \gamma_k^2 \mathbb{E} \left\|    \frac{1}{B} \sum\limits_{i \in I_t}  \nabla_k f_i(w_t) - \nabla_k f(w_t ) + \nabla_k f(w_t) \right\|_2^2 \nonumber \\
	&= &  \sum\limits_{k=1}^K \gamma_k^2 \mathbb{E} \left\|    \frac{1}{B} \sum\limits_{i \in I_t}  \nabla_k f_i(w_t) - \nabla_k f(w_t )  \right\|_2^2 +   \sum\limits_{k=1}^K \gamma_k^2 \mathbb{E} \left\|   \nabla_k f(w_t) \right\|_2^2 \nonumber \\
	&\leq &   \sum\limits_{k=1}^K  \frac{M\gamma_k^2 }{B}   +  \sum\limits_{k=1}^K   \left( \gamma_k^2 +  \frac{M_k\gamma_k^2 }{B} \right)   \mathbb{E} \left\|   \nabla_k f(w_t) \right\|_2^2,
	\label{iq_3002}
	\end{eqnarray}
	where the third equality follows from $\mathbb{E}\left<  \frac{1}{B} \sum\limits_{i \in |I_t|}  \nabla_k f_i(w_t) - \nabla_k f(w_t ), \nabla_k f(w_t)  \right>=0$ and the last inequality follows from Assumption \ref{ass_bd}. Combining inequalities (\ref{iq_3001}) and (\ref{iq_3002}), we have the upper bound of $ \mathbb{E} \left\|(p_{t})_k  \right\|_2^2$ as follows:
	\begin{eqnarray}
	\mathbb{E} \left\|p_{t} \right\|_2^2 &\leq & \sum\limits_{k=1}^K  \frac{\beta^4 \gamma_{k}^2 \Gamma_{t-1}}{(1-\beta)^2 }  \left(  \frac{M}{B}  \sum\limits_{j=0}^{t-1} {\beta^{j}}  + \left( 1+ \frac{M_k}{B}\right) \sum\limits_{j=0}^{t-1} {\beta^{j}}  \mathbb{E} \left\|   \nabla_k f(w_{t-1-j}) \right\|_2^2 \right)    \nonumber \\
	& \leq  &  \sum\limits_{k=1}^K   \frac{\beta^4\gamma_{k}^2 M}{(1-\beta)^4 B } +    \sum\limits_{k=1}^K\left( 1+ \frac{M_k}{B}\right) \frac{\beta^4 \gamma_{k}^2 }{(1-\beta)^3 }  \sum\limits_{j=0}^{t-1} {\beta^{j}}   \mathbb{E} \left\|   \nabla_k f(w_{t-1-j}) \right\|_2^2,
	\label{iq_3003}
	\end{eqnarray}
	where the last inequality follows from $\Gamma_{t-1}=\sum\limits_{j=0}^{t-1} {\beta^{j}}  = \frac{1- \beta^t}{1-\beta} \leq \frac{1}{1-\beta} $. Summing inequality (\ref{iq_3003}) from $t=0$ to $T-1$, we have:
	\begin{eqnarray}
	\sum\limits_{t=0}^{T-1}  \mathbb{E} \left\|p_{t}  \right\|_2^2& \leq  &    \sum\limits_{k=1}^K \frac{\beta^4\gamma_{k}^2MT}{(1-\beta)^4 B }  +  \sum\limits_{k=1}^K  \left( 1+ \frac{M_k}{B}\right) \frac{\beta^4\gamma_{k}^2}{(1-\beta)^3 }  \sum\limits_{t=0}^{T-1} \sum\limits_{j=0}^{t-1} {\beta^{j}}   \mathbb{E} \left\|   \nabla_k f(w_{t-1-j}) \right\|_2^2 \nonumber \\
	& =  &   \sum\limits_{k=1}^K \frac{\beta^4\gamma_{k}^2MT}{(1-\beta)^4 B }   + \sum\limits_{k=1}^K \left( 1+ \frac{M_k}{B}\right)   \frac{\beta^4\gamma_{k}^2}{(1-\beta)^3 }  \sum\limits_{t=0}^{T-1} \mathbb{E} \left\|   \nabla_k f(w_{t}) \right\|_2^2 \sum\limits_{j=t}^{T-1}  {\beta^{T-1-j}}  \nonumber \\
	&\leq &  \sum\limits_{k=1}^K \frac{\beta^4\gamma_{k}^2MT}{(1-\beta)^4 B }   +  \sum\limits_{k=1}^K  \left( 1+ \frac{M_k}{B}\right) \frac{\beta^4\gamma_{k}^2}{(1-\beta)^4 } \sum\limits_{t=0}^{T-1}  \mathbb{E} \left\|   \nabla_k f(w_{t}) \right\|_2^2,
	\label{iq_6002}
	\end{eqnarray}
	where the last inequality follows from $\sum\limits_{j=t}^{T-1}  {\beta^{T-1-j}} \leq \frac{1}{1-\beta}$ for any $t \in \{0, 1, ...,T-1\}$.
	
	$\hfill\blacksquare$
\end{proof}

\textbf{Proof of Theorem \ref{them2}}
\begin{proof}
	Following Lemma \ref{lem2_1} and summing  inequality (\ref{iq_5005}) from $t=0$ to $T-1$, we have:
	\begin{eqnarray} 
	f_{\inf}  &\leq & f(w_0) - \sum\limits_{k=1}^K \frac{\gamma_{k}}{2(1-\beta)}    \bigg( 1- \frac{L_{k} \gamma_{k} }{1-\beta} - \frac{L_{k} \gamma_{k}M_{k} }{(1-\beta)B}  - \frac{2L_{g}^2 \gamma_{k}    M_{k}  }{(1-\beta)^2  KB}  \sum\limits_{k=1}^{K}   \gamma_{k}  -   \frac{2L_{g}^2 \gamma_{k} }{(1-\beta)^2K}  \sum\limits_{k=1}^{K}   \gamma_{k}    \bigg) \sum\limits_{t=0}^{T-1} \mathbb{E} \left\|\nabla_{k} f(w_t)  \right\|_2^2  \nonumber \\
	&&  +  \frac{ M L_{g}^2 T }{(1-\beta)^3  KB}  \sum\limits_{k=1}^{K}  \gamma_{k}   \sum\limits_{k=1}^{K}   \gamma_{k}^2 +  \sum\limits_{k=1}^{K}  \frac{L_{k} \gamma_{k}^2 M T}{2(1-\beta)^2 B} +    \sum\limits_{k=1}^{K}  \frac{ L_{g}^2 \gamma_{k}}{(1-\beta)K}  \sum\limits_{t=0}^{T-1}  \mathbb{E} \left\|p_{t}  \right\|_2^2.
	\label{iq_6001}
	\end{eqnarray}
	where we have  $z_t = w_0$ and $f(z_T) \geq f_{\inf}$.
	According to Lemma \ref{lem2_2} and inputting (\ref{iq_6002}) in inequality (\ref{iq_6001}), the following inequality is satisfied that:
	\begin{eqnarray}
	f_{\inf}  &\leq & f(w_0) - \sum\limits_{k=1}^K \frac{\gamma_{k}}{2(1-\beta)}    \bigg( 1- \frac{L_{k} \gamma_{k} }{1-\beta} - \frac{L_{k} \gamma_{k}M_{k} }{(1-\beta)B} - \frac{2 L_g^2\beta^4 \gamma_{k} \sum\limits_{k=1}^K\gamma_{k}}{(1-\beta)^4K } - \frac{2L_g^2\beta^4 \gamma_{k} \sum\limits_{k=1}^K\gamma_{k} M_k}{(1-\beta)^4KB }  \nonumber \\
	&& - \frac{2L_{g}^2 \gamma_{k}    M_{k}  }{(1-\beta)^2  KB}  \sum\limits_{k=1}^{K}   \gamma_{k}  -   \frac{2L_{g}^2 \gamma_{k} }{(1-\beta)^2K}  \sum\limits_{k=1}^{K}   \gamma_{k}    \bigg) \sum\limits_{t=0}^{T-1} \mathbb{E} \left\|\nabla_{k} f(w_t)  \right\|_2^2  \nonumber \\
	&&  +  \frac{ M L_{g}^2 T }{(1-\beta)^3  KB}  \sum\limits_{k=1}^{K}  \gamma_{k}   \sum\limits_{k=1}^{K}   \gamma_{k}^2 +  \sum\limits_{k=1}^{K}  \frac{L_{k} \gamma_{k}^2 M T}{2(1-\beta)^2 B} +    \sum\limits_{k=1}^{K}  \frac{ L_{g}^2 \gamma_{k} \beta^4 MT}{(1-\beta)^5KB}  \sum\limits_{k=1}^K {\gamma_{k}^2}.
	\label{iq_6003}
	\end{eqnarray}
	Defining $\kappa_k = \frac{L_g}{L_k} \leq \kappa $, if $\gamma_k$ satisfies following inequalities:
	\begin{eqnarray}
	\frac{L_k \gamma_{k}}{1-\beta } & \leq & \frac{1}{8}, \\
	\frac{L_k \gamma_{k}M_k}{(1-\beta) B} & \leq & \frac{1}{8}, \\
	\frac{2L_g^2 \beta^4  \gamma_{k} \sum\limits_{k=1}^K \gamma_{k}}{(1-\beta)^4K} &\leq & \frac{1}{8}, \\
	\frac{2L_g^2 \beta^4 \gamma_{k} \sum\limits_{k=1}^K M_k\gamma_{k}}{(1-\beta)^4KB} &\leq & \frac{1}{8}, \\
	\frac{2L_{g}^2 \gamma_{k}    M_{k}  }{(1-\beta)^2  KB}  \sum\limits_{k=1}^{K}   \gamma_{k}  &\leq & \frac{1}{8}, \\
	\frac{2L_{g}^2 \gamma_{k} }{(1-\beta)^2K}  \sum\limits_{k=1}^{K}   \gamma_{k}  &\leq & \frac{1}{8},
	\end{eqnarray}
	which is equivalent to $\gamma_k \leq \min\left\{\frac{(1-\beta)}{8L_k}, \frac{(1-\beta)B}{8 L_kM_k} \right\}$ and : 
	$
	\frac{1}{K}	\sum\limits_{k=1}^K \gamma_k \leq \min\left\{ \frac{(1-\beta)^2}{4\beta^2 L_g }, \frac{(1-\beta)^2 \sqrt{B}}{4\beta^2 L_g \sqrt{ M_g} },  \frac{(1-\beta) \sqrt{B} }{4L_g  \sqrt{M_g}  }  , \frac{(1-\beta)}{4 L_g}   \right\}$.
	It holds that:
	\begin{eqnarray}
	\sum\limits_{t=0}^{T-1} 	\sum\limits_{k=1}^K    \frac{\gamma_{k}}{8(1-\beta)}  \mathbb{E} \left\|\nabla_{k} f(w_t)  \right\|_2^2 &\leq  & f(w_0) - f_{\inf}  + \frac{MT}{(1-\beta)^2B} \bigg(\frac{1}{2} \sum\limits_{k=1}^K L_k \gamma_k^2      \nonumber \\
	&&  +  \frac{L_g^2}{(1-\beta)K} \sum\limits_{k=1}^K \gamma_k \sum\limits_{k=1}^K \gamma_k^2  +  \frac{\beta^4 L_g^2}{ (1-\beta)^3K } \sum\limits_{k=1}^K \gamma_k \sum\limits_{k=1}^K \gamma_k^2   \bigg) \nonumber \\
	&\leq & f(w_0) - f_{\inf}  \nonumber \\
	&&+ \frac{MT}{(1-\beta)^2B} \bigg(\frac{1}{2} \sum\limits_{k=1}^K L_k \gamma_k^2     +  \frac{1}{4}  \sum\limits_{k=1}^K L_g \gamma_k^2   +  \frac{1}{4(1-\beta)}  \sum\limits_{k=1}^KL_g \gamma_k^2   \bigg). 
	\end{eqnarray}
	Let $\gamma_{k} = \frac{\gamma}{L_k}$ and dividing both sides by $\sum\limits_{k=1}^K\frac{T}{8(1-\beta)}\gamma_{k}$, it holds that:
	\begin{eqnarray}
	\frac{1}{T}  \sum\limits_{t=0}^{T-1} \sum\limits_{k=1}^K q_k  \mathbb{E} \left\|\nabla_{k} f(w_t)  \right\|_2^2 &\leq&  \frac{8(1-\beta)(f(w_0) - f_{\inf} ) }{T \gamma\sum\limits_{k=1}^K \frac{1}{L_k}} \nonumber \\
	&&   + \frac{M \gamma}{(1-\beta)B} \left(4 +   2\kappa +  \frac{2 \kappa }{ (1-\beta) }  \right).
	\end{eqnarray}
	where $q_k = \frac{\frac{1}{L_k}}{\sum\limits_{k=1}^K \frac{1}{L_k} }$.
	
	$\hfill\blacksquare$
\end{proof}

\textbf{Proof of Corollary \ref{cor2_1}}
\begin{proof}
	suppose $\frac{1-\beta}{8L_k}$ dominates the upper bound of $\gamma_k$, if:
	\begin{eqnarray}
	\gamma & = & \min \left\{ \frac{1-\beta}{8}, \sqrt{ \frac{B (f(w_0)-f_{\inf} )  }{TM \sum\limits_{k=1}^K \frac{1}{L_k} }  } \right\}
	\end{eqnarray}
	we have:
	\begin{eqnarray}
	\min\limits_{t \in \{0,...,T-1\}} \sum\limits_{k=1}^K q_k \mathbb{E} \left\|   \nabla_k f(w_t)  \right\|_2^2  &\leq &  \frac{8(f(w_0)-f_{\inf} )}{T \sum\limits_{k=1}^K \frac{1}{L_k} } \max\left\{\frac{8}{1-\beta}, \sqrt{ \frac{TM \sum\limits_{k=1}^K \frac{1}{L_k} }{B (f(w_0)-f_{\inf} )  }  }  \right\}  \nonumber \\
	&& + \frac{M }{(1-\beta)B} \left(4 +   2\kappa +  \frac{2 \kappa }{ (1-\beta) }  \right) \sqrt{ \frac{B (f(w_0)-f_{\inf} )  }{TM \sum\limits_{k=1}^K \frac{1}{L_k} }  } \nonumber \\
	&\leq &  \frac{64(f(w_0)-f_{\inf} )}{(1-\beta)T \sum\limits_{k=1}^K \frac{1}{L_k} }  \nonumber\\
	&& + \left( 8 + \frac{ 1}{(1-\beta)} \left(4 +   2\kappa +  \frac{2 \kappa }{ (1-\beta) }  \right) \right) \sqrt{ \frac{M (f(w_0)-f_{\inf} )  }{TB \sum\limits_{k=1}^K \frac{1}{L_k}  }  }. 
	\end{eqnarray}
	where the left side follows from $	\min\limits_{t \in \{0,...,T-1\}} \sum\limits_{k=1}^K q_k  \mathbb{E} \left\|\nabla_{k} f(w_t)  \right\|_2^2 \leq  		\frac{1}{T}  \sum\limits_{t=0}^{T-1} \sum\limits_{k=1}^K q_k  \mathbb{E} \left\|\nabla_{k} f(w_t)  \right\|_2^2 $.
	we complete the proof.
	
	$\hfill\blacksquare$
\end{proof}

\section{Neural Network Architectures}

\begin{figure*}[h]
	\centering
	\begin{subfigure}[b]{0.45\textwidth}
		\centering
		\includegraphics[width=1.9in]{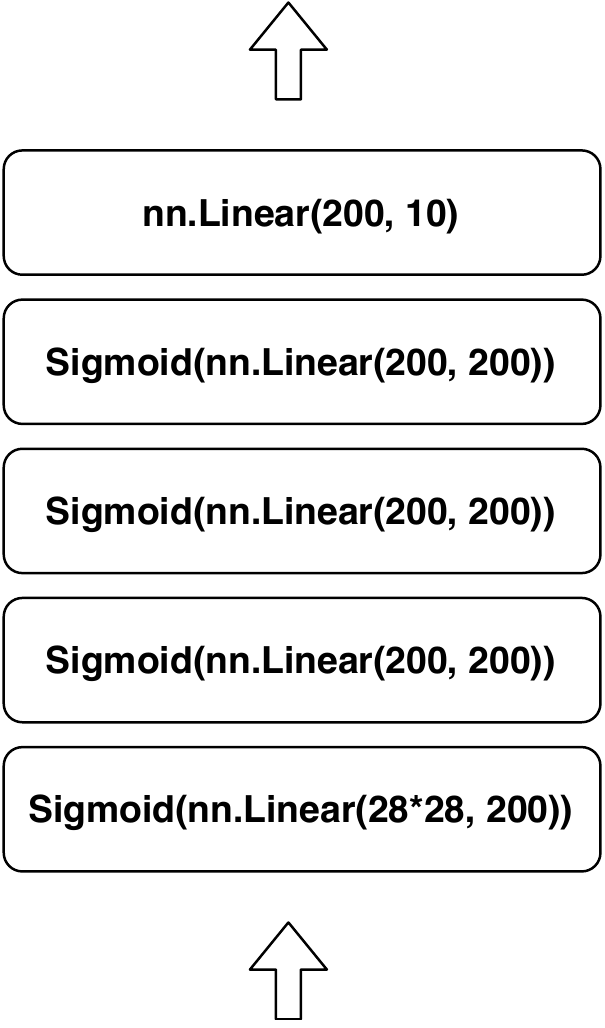}
		\caption{5-layer FCN}
	\end{subfigure}
	\begin{subfigure}[b]{0.45\textwidth}
		\centering
		\includegraphics[width=1.9in]{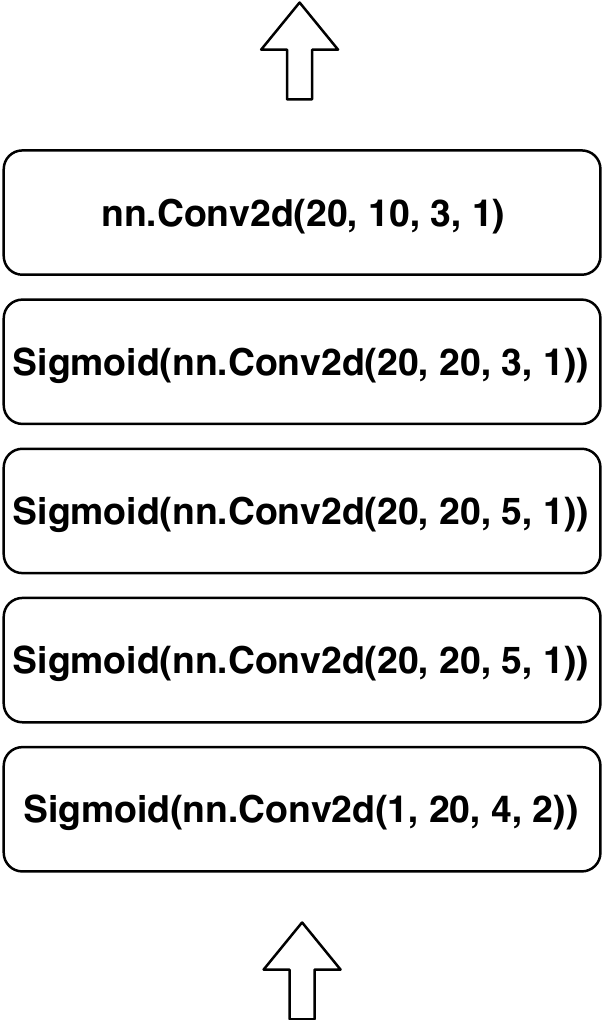}
		\caption{5-layer CNN}
	\end{subfigure}
	\label{fig::archs}
\end{figure*}

\end{document}